\documentclass[runningheads]{llncs}
\usepackage[T1]{fontenc}
\usepackage{graphicx}
\usepackage{amsmath,amsfonts,bm}

\def\eqref#1{equation~\ref{#1}}

\def\1{\bm{1}}

\DeclareMathAlphabet{\mathsfit}{\encodingdefault}{\sfdefault}{m}{sl}
\SetMathAlphabet{\mathsfit}{bold}{\encodingdefault}{\sfdefault}{bx}{n}

\usepackage{hyperref}
\usepackage{url}

\usepackage{booktabs}
\usepackage{bm}
\usepackage{mathrsfs}
\usepackage{multirow}
\usepackage{balance}
\usepackage{amsmath}
\usepackage{amstext}
\usepackage[ruled,vlined,linesnumbered]{algorithm2e}
\usepackage{overpic}
\usepackage{color}
\usepackage{enumitem}
\usepackage{fancyvrb}
\usepackage{graphicx}
\usepackage{caption}
\usepackage{multirow}
\usepackage{subfigure}
\usepackage{multicol}
\usepackage{extarrows}
\usepackage{amssymb}
\usepackage{diagbox}
\usepackage{makecell}
\usepackage[mathscr]{eucal}
\usepackage{hyperref}
\usepackage[switch]{lineno}
\usepackage{setspace}
\usepackage{wrapfig}

\definecolor{gray}{RGB}{150, 150, 150}
\definecolor{orange}{RGB}{255, 97, 0}
\definecolor{darkblue}{RGB}{25, 25, 112}

\def\Algnamea{\underline{A}ctive \underline{P}enalization \underline{O}f \underline{D}iscrimination}
\def\Algnameb{Active Penalization Of Discrimination}
\def\Algnameabbr{APOD}

\title{Mitigating Algorithmic Bias with Limited Annotations}

\author{Guanchu Wang\inst{1} \and Mengnan Du\inst{2} \and Ninghao Liu\inst{3} \and Na Zou\inst{4} \and Xia Hu\inst{1}}
\authorrunning{Guanchu Wang \and Mengnan Du \and Ninghao Liu \and Na Zou \and Xia Hu}
\institute{Department of Computer Science, Rice University, USA \and Department of Computer Science and Engineering, New Jersey Institute of Technology, USA \and Department of Department of Computer Science, University of Georgia, USA \and Department of Engineering Technology and Industrial Distribution, Texas A\&M University, USA}

\begin{document}

\maketitle

\begin{abstract}

Existing work on fairness modeling commonly assumes that sensitive attributes for all instances are fully available, which may not be true in many real-world applications due to the high cost of acquiring sensitive information. When sensitive attributes are not disclosed or available, it is needed to manually annotate a small part of the training data to mitigate bias. However, the skewed distribution across different sensitive groups preserves the skewness of the original dataset in the annotated subset, which leads to non-optimal bias mitigation. To tackle this challenge, we propose \Algnamea{}~(\Algnameabbr{}), an interactive framework to guide the limited annotations towards maximally eliminating the effect of algorithmic bias. The proposed \Algnameabbr{} integrates discrimination penalization with active instance selection to efficiently utilize the limited annotation budget, and it is theoretically proved to be capable of bounding the algorithmic bias. According to the evaluation on five benchmark datasets, \Algnameabbr{} outperforms the state-of-the-arts baseline methods under the limited annotation budget, and shows comparable performance to fully annotated bias mitigation, which demonstrates that \Algnameabbr{} could benefit real-world applications when sensitive information is limited. 
The source code of the proposed method is available at:
\href{https://anonymous.4open.science/r/APOD-fairness-4C02}{{\color{darkblue} \texttt{https://anonymous.4open.science/r/APOD-fairness-4C02}}}.

\end{abstract}

\section{Introduction}
\label{sec:intro}

Although deep neural networks (DNNs) have been demonstrated with great performance in many real-world applications, it shows discrimination towards certain groups or individuals ~\cite{caton2020fairness,tolmeijer2020implementations,rajkomar2018ensuring,bobadilla2020deepfair}, especially in high-stake applications, e.g., loan approvals~\cite{steel2010web}, policing~\cite{goel2016precinct}, targeted advertisement~\cite{sweeney2013discrimination}, college admissions~\cite{zimdars2010fairness}, or criminal risk assessments~\cite{angwin2016there}.
Social bias widely exists in many real-world data~\cite{mehrabi2021survey,chen2018my,li2019repair,chuang2021fair}. 
For example, the Adult dataset~\cite{UCI:2007} contains significantly more low-income female instances than males.
Recent studies revealed that training a DNN model on biased data may inherit and even amplify the social bias and lead to unfair predictions in downstream tasks~\cite{dwork2012fairness,creager2019flexibly,sun2019mitigating,kusner2017counterfactual,dai2021say}.


The problem of bias mitigation is challenging due to the skewed data distribution~\cite{hashimoto2018fairness,azzalini1996multivariate,azzalini2005skew} across different demographic groups. For example, in the Adult dataset, instances of female with high income are significantly less than the ones with low income~\cite{UCI:2007}. Also, in the German credit dataset, the majority of people younger than 35 show a bad credit history~\cite{UCI:2017}.
The effect of the skewed distribution on model fairness is illustrated in a binary classification task (e.g. positive class denoted as gray + and $\bullet$, negative class as red \textcolor{red}{+} and \textcolor{red}{$\bullet$}) with two sensitive groups (e.g. group 0 denoted as + and \textcolor{red}{+}, group 1 as $\bullet$ and \textcolor{red}{$\bullet$}) shown in Figure~\ref{fig:unfair_vs_rs_vs_APD}.
In Figure~\ref{fig:unfair_vs_rs_vs_APD}~(a), the positive instances~(+) are significantly less than negative instances~(\textcolor{red}{+}) in group 0, which leads to a classification boundary deviating from the fair one.
Existing work on fairness modeling can be categorized into two groups with or without sensitive attributes~\cite{du2020fairness,kleinberg2018algorithmic}. The first group relied on full exposure of sensitive attributes in training data, such as Fair Mixup~\cite{chuang2021fair}, FIFA~\cite{deng2022fifa}, Fair Rank~\cite{mehrotrafair}, and Group DRO~\cite{sagawa2019distributionally}. However, the sensitive information may not be disclosed or available in some real world scenarios~\cite{zhao2021you,kallus2021assessing}, and the cost of annotating sensitive attributes by experts is high~\cite{anahideh2020fair}, which leads to the limited applications of this group of work to the real-world scenarios.



The second group of work formulates the fairness without dependency on sensitive information, such as SS-FRL~\cite{chaiself}, FKD~\cite{chaifairness}, and LfF~\cite{nam2020learning}.  
However, those works rely on heuristic clustering of training instances to form potential demographic groups for the bias mitigation, which may deteriorate the fairness performance to some extent~\cite{wang2020robust}. 
To tackle the issue, some work involves the human expert providing a partially annotated dataset for the bias mitigation~\cite{anahideh2020fair}.
However, only a small portion of the dataset is annotated due to the limitation of human labor efforts.
An intuitive solution is to randomly select a small portion of instances for annotation and target semi-supervised bias mitigation~\cite{zhangfairness}.
However, as shown in Figure~\ref{fig:unfair_vs_rs_vs_APD}~(b), the randomly selected instances will follow the same skewed distribution across sensitive groups, which still preserves the bias information in the classifier. 
In such a manner, it is highly likely to achieve a non-optimal solution, which is fair only on the annotated dataset but not the entire dataset. 
Therefore, it is needed to have a unified framework, which integrates the selection of a representative subset for annotation with model training towards the global fairness~\cite{abernethy2022active}, as shown in Figure~\ref{fig:unfair_vs_rs_vs_APD}~(c).

\begin{figure*}[t]
\setlength{\abovecaptionskip}{1mm}
\setlength{\belowcaptionskip}{-5mm}
    \centering
    \subfigure[]{
    \centering
		\begin{minipage}[t]{0.31\linewidth}
			\includegraphics[width=0.99\linewidth]{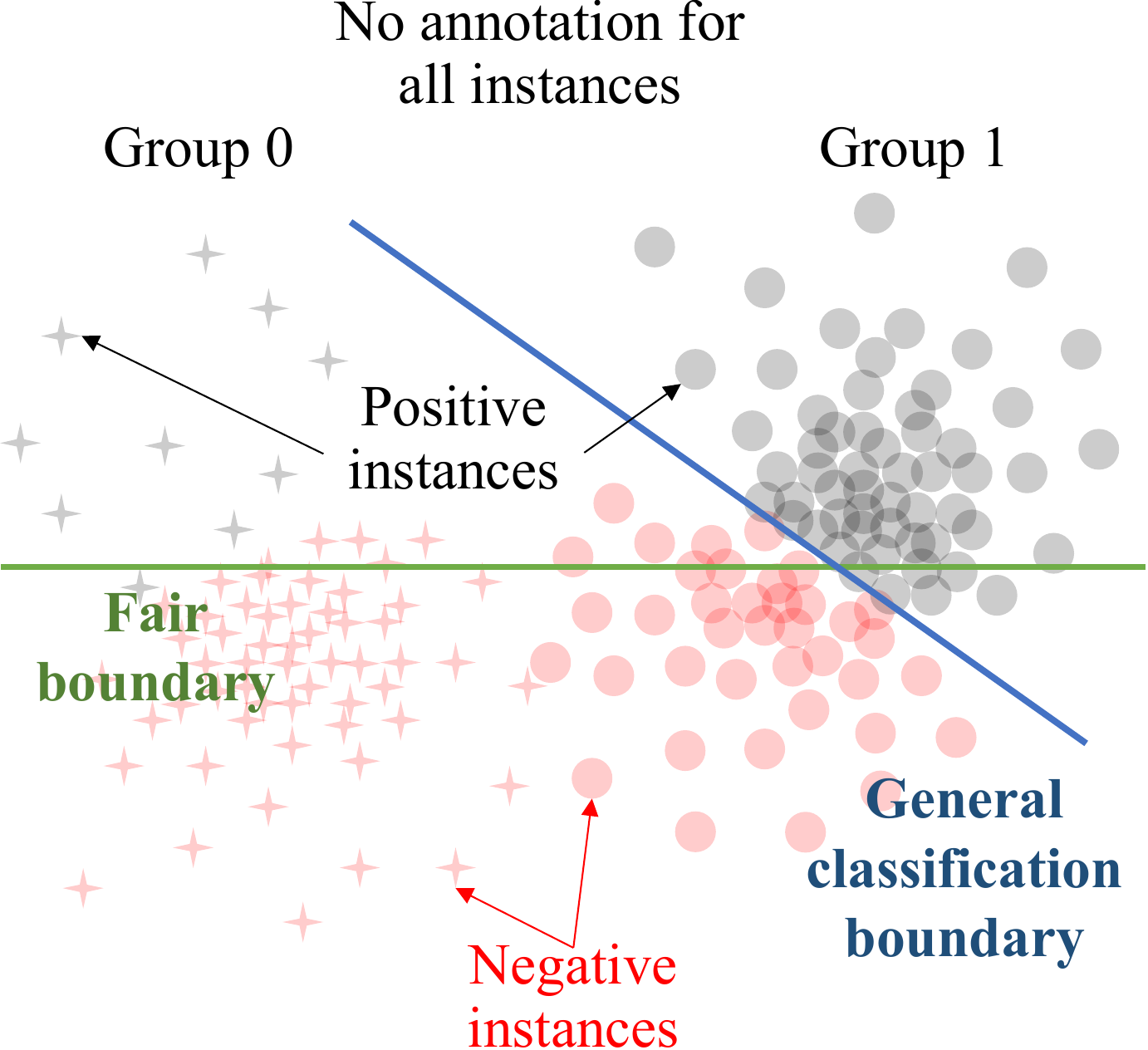}
		\end{minipage}%
	}
	\subfigure[]{
	\centering
		\begin{minipage}[t]{0.31\linewidth}
			\includegraphics[width=0.99\linewidth]{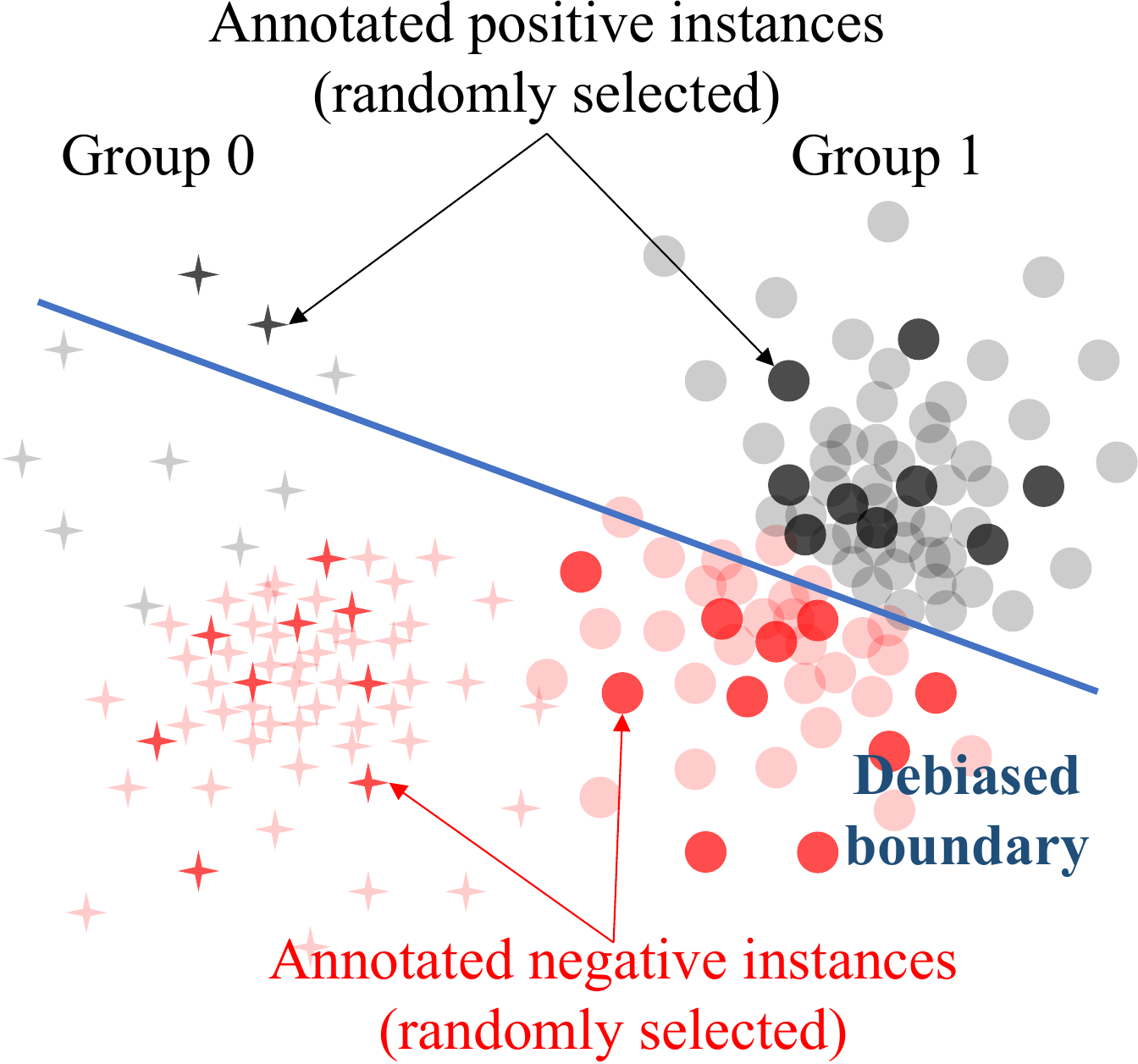}
		\end{minipage}%
	}
	\subfigure[]{
	\centering
		\begin{minipage}[t]{0.32\linewidth}
			\includegraphics[width=1.0\linewidth]{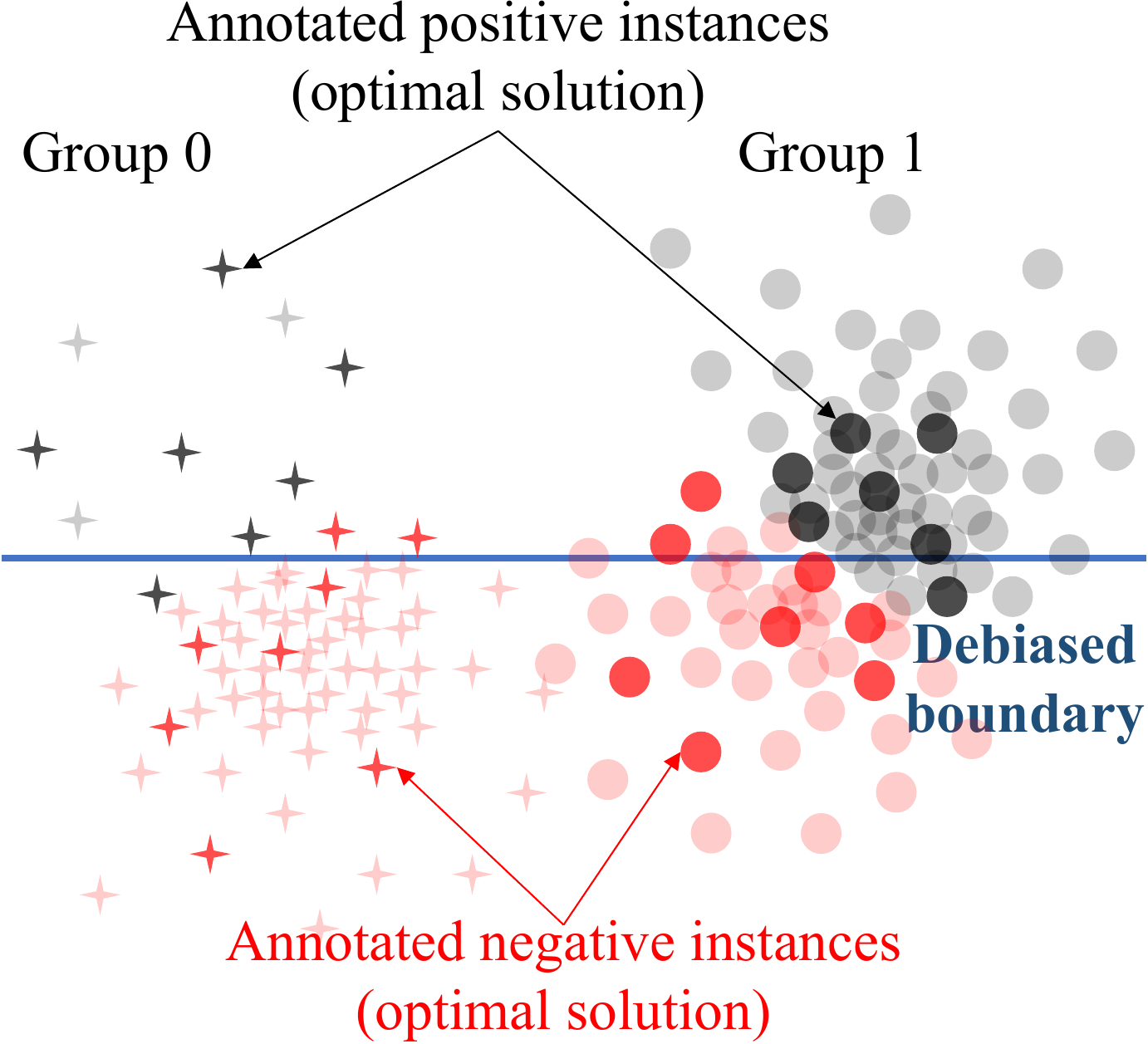}
		\end{minipage}
	}
    \caption{\small (a) The general classification boundary without bias mitigation deviates from the fair boundary due to the skewed distribution across four underlying subgroups~(i.e. +, \textcolor{red}{+}, $\bullet$ and \textcolor{red}{$\bullet$}). 
    (b) The annotation budget is set as 30.
    The randomly annotated data subset follows the same skewed distribution across the subgroups. The classification model is still unfair on the entire dataset.
    (c) With the same annotation budget, the optimal solution should select a more representative subset, which mitigates algorithmic bias on the entire dataset.
    }
    \label{fig:unfair_vs_rs_vs_APD}
\end{figure*}

In this work, we propose \Algnamea{} (\Algnameabbr{}), a novel interactive framework which integrates the penalization of discrimination with active instance selection, for bias mitigation in the real-world scenarios where sensitive information is limited.
Specifically, \Algnameabbr{} iterates between the model debiasing and active instance selection to gradually approach the global fairness.
For debiasing the model, \Algnameabbr{} enables bias penalization in an end-to-end manner via adopting a fairness regularizer.
In the active instance selection, an annotated data subset is constructed via recursive selections of representative data instances from the subgroup where the model shows the worst performance, such that it can maximally expose the existing bias of the model for subsequent debiasing.
Finally, we provide theoretical and experimental analysis to demonstrate the effectiveness of \Algnameabbr{}.
Overall, the contributions of this work are summarized as follows:
\begin{itemize}[leftmargin=15pt, topsep=1mm]
    
    \item[$\bullet$]  We propose an interactive framework \Algnameabbr{} to integrate the bias mitigation with efficient active instance selection when the annotation of sensitive attributes is very limited. 
    
    
    \item[$\bullet$]  We propose the relaxed reformulation of the fairness objective, and theoretically prove that \Algnameabbr{} could improve model fairness via bounding the relaxed fairness metric.
     
    
    \item[$\bullet$]  The effectiveness of \Algnameabbr{} is thoroughly demonstrated by the experiments on five benchmark datasets, which shows \Algnameabbr{} is competitive with state-of-the-art methods using fully disclosed sensitive attributes.
    
\end{itemize}

\section{Preliminaries}

In this section, we first introduce the notations used in this work, and give the problem definition of bias mitigation in the active scenario. 
Then, we introduce the fairness metrics.




\subsection{Notation and Problem Definition}

Without loss of generality, we follow the existing work~\cite{chuang2021fair,zhang2018mitigating,lahoti2020fairness} to consider a classification task in this work.
Specifically, we aim to learn a DNN classifier $f$ with the input feature $\boldsymbol{x} \in \mathcal{X}$, label $y \in \mathcal{Y} = \{ 0,1 \}$ and sensitive attribute $a \in \mathscr{A} = \{ 0,1 \}$, where $\mathcal{X}$ and $\mathcal{Y}$ denote the feature and label space, respectively.
The instances with sensitive attribute $A=0$ and $A=1$ belong to the unprivileged and privileged groups, respectively. 
Let $\mathscr{D} = \{ (\boldsymbol{x}_i, y_i) \mid 1 \leq i \leq N \}$ denote the entire dataset, which consists of the annotated set $\mathcal{S} = \{ (\boldsymbol{x}_i, y_i, a_i) \}$ and unannotated set $\mathscr{U} = \{ (\boldsymbol{x}_i, y_i) \}$, i.e., the value of the sensitive attribute is known for instances in $\mathcal{S}$, but it is unknown for instances in $\mathscr{U}$.
The proposed interactive bias mitigation is illustrated in Figure~\ref{fig:sys_config}~(a).
Specifically, in each iteration, an instance $(\boldsymbol{x}^*, y^*)$ is selected from unannotated dataset $\mathscr{U}$ for human experts; 
the experts essentially do the job of mapping $\mathcal{X} \times \mathcal{Y} \to \mathcal{X} \times \mathcal{Y} \times \mathscr{A}$, by providing the annotation of sensitive attribute $a^*$ for the selected instance $(\boldsymbol{x}^*, y^*)$. 
After that, the classifier is updated and debiased using the partially annotated dataset including the newly annotated instance $(\boldsymbol{x}^*, y^*, a^*)$, 
where the new classifier will then be involved for the instance selection in the next iteration.
This loop terminates if the human-annotation budget is reached.


Such an active scenario to debias $f$ is time-consuming for deep neural networks, due to the retraining of $f$ in each iteration. 
To improve the efficiency of learning, the classifier $f$ is split into body $f_b: \mathcal{X} \!\to\! \mathbb{R}^M$ and head $f_h: \mathbb{R}^M \!\to\! \mathbb{R}^{|\mathcal{Y}|}$, where the body $f_b$ denotes the first several layers, and the head $f_h$ denotes the remaining layers of the classifier such that $\hat{y}_i \!=\! \mathop{\arg\max} \{ f_h( f_b(\boldsymbol{x}_i | \theta_b) | \theta_h)\}$.
The body $f_b$ learns the instance embedding $\boldsymbol{h}_i = f_b(\boldsymbol{x}_i | \theta_h)$, where $\boldsymbol{h}_i \in \mathbb{R}^M$ denotes the embedding of $\boldsymbol{x}_i$, and $M$ denotes the dimension of embedding space. 
The head $f_h$ contributes to fair classification via having $\hat{y}_i = \mathop{\arg\max} \{f_h(\boldsymbol{h}_i | \theta_h)\}$, where $f_h(\boldsymbol{h}_i | \theta_h) \in \mathbb{R}^{|\mathcal{Y}|}$ and $\hat{y}_i \in \mathcal{Y}$.
Instead of updating the entire classifier, the classifier body $f_b$ is pretrained and fixed during the bias mitigation, where $f_b$ is pretrained to minimize the cross-entropy loss without annotations of sensitive attributes.
In such a manner, the mitigation of unfairness relies on debiasing the classifier head $f_h$.
This strategy with a fixed classifier body during the bias mitigation has been proved to be effective enough in existing works~\cite{du2021fairness,slack2020fairness}.



\subsection{Fairness Evaluation Metrics}
\label{sec:fairness_evaluation}
In this work, we follow existing work~\cite{mehrabi2021survey,du2021fairness} to consider two metrics to evaluate fairness: Equality of Opportunity~\cite{hardt2016equality,verma2018fairness} and Equalized Odds~\cite{romano2020achieving,verma2018fairness}. These two metrics are measured based on the true positive rate $\text{TPR}_{A=a} \!=\! \mathbb{P} (\hat{Y} \!=\! 1 | A \!=\! a, Y \!=\! 1)$ and the false positive rate $\text{FPR}_{A=a} \!=\! \mathbb{P} (\hat{Y} \!=\! 1 | A \!=\! a, Y \!=\! 0)$ for $a \!\in\! \mathcal{A}$.

\paragraph{\textbf{Equality of Opportunity}} requires the unprivileged group ($A=0$) and privileged groups ($A=1$) have equal probability of an instance from the positive class being assigned to positive outcome, which is defined as $\mathbb{P}(\hat{Y}=1 | A=0, Y=1) = \mathbb{P}(\hat{Y}=1 | A=1, Y=1)$.
In this work, we apply EOP given as follows to evaluate Equality of Opportunity,
\begin{equation}
\begin{aligned}
\label{eq:EOP}
    \text{EOP} = \frac{\text{TPR}_{A=0}}{\text{TPR}_{A=1}} = \frac{\mathbb{P}(\hat{Y}=1 \mid A=0, Y=1)}{\mathbb{P}(\hat{Y}=1 \mid A=1, Y=1)}.
\end{aligned}
\end{equation}


\paragraph{\textbf{Equalized Odds}} expects favorable outcomes to be independent of the sensitive attribute, given the ground-truth prediction, which can be formulated as $\mathbb{P}(\hat{Y} \!=\! 1 | A \!=\! 0, Y \!=\! y) \!=\! \mathbb{P}(\hat{Y} \!=\! 1 | A \!=\! 1, Y \!=\! y)$ for $y \!\in\! \mathcal{Y}$.
To evaluate Equalized Odds, $\Delta \text{EO}$ combines the difference of TPR and FPR across two sensitive groups as 
\begin{equation}
\begin{aligned}
\label{eq:EO}
    \Delta \text{EO} = \Delta \text{TPR} + \Delta \text{FPR},
\end{aligned}
\end{equation}
where $\Delta \text{TPR} = \text{TPR}_{A=0} - \text{TPR}_{A=1}$ and $\Delta \text{FPR} = \text{FPR}_{A=0} - \text{FPR}_{A=1}$. 
Under the above definitions, EOP close to 1 and $\Delta \text{EO}$ close to 0 indicate fair classification results.

\begin{figure}[t]
\setlength{\abovecaptionskip}{-2mm}
\setlength{\belowcaptionskip}{-2mm}
    \centering
    \begin{minipage}{0.4\linewidth}
    \centering
    \subfigure[]{
    \centering
    \includegraphics[width=1.0\textwidth]{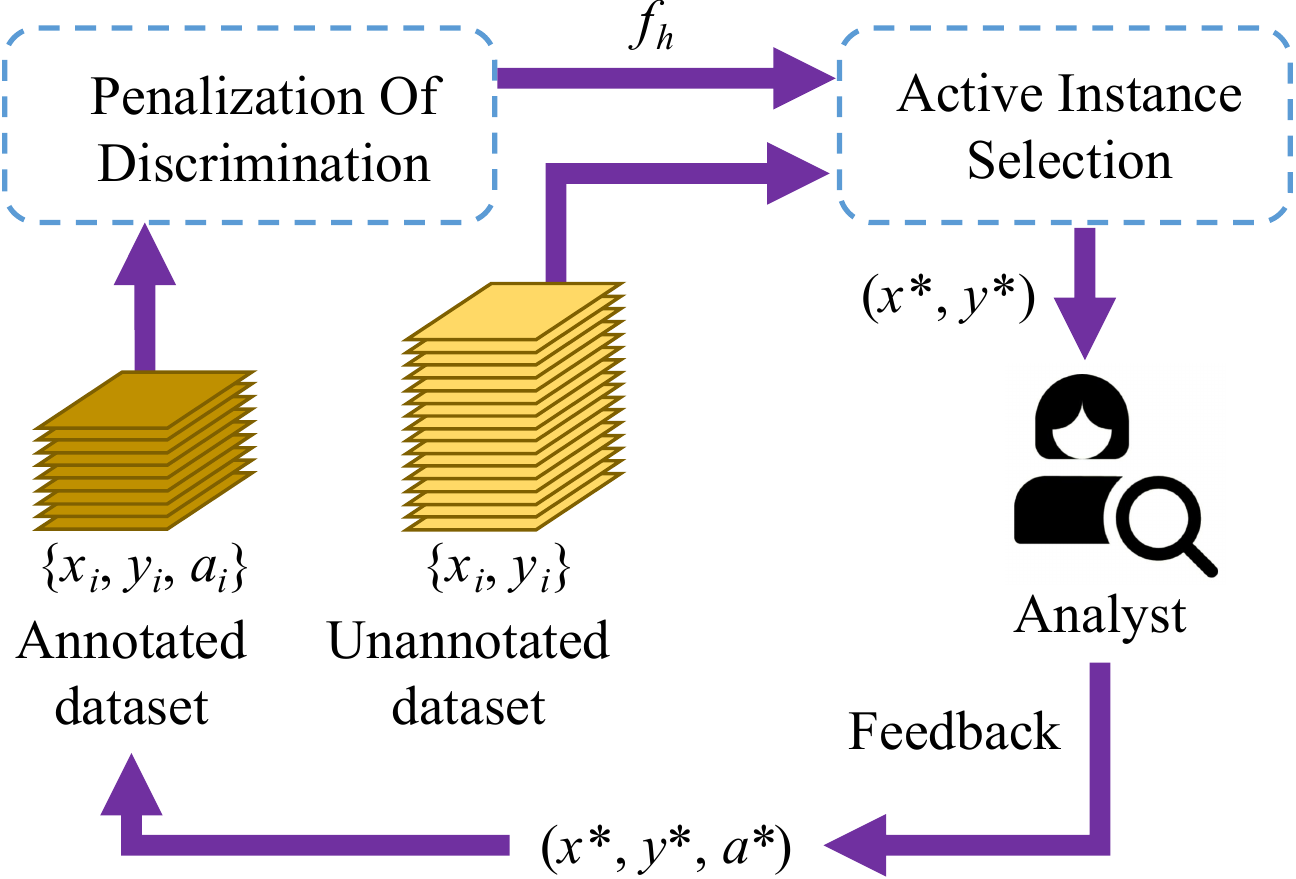}}
    \end{minipage}
    \hspace{5pt}
    \text{\qquad}
    \begin{minipage}{0.23\linewidth}
    \centering
    \subfigure[]{
    \centering
    \includegraphics[width=1.0\textwidth]{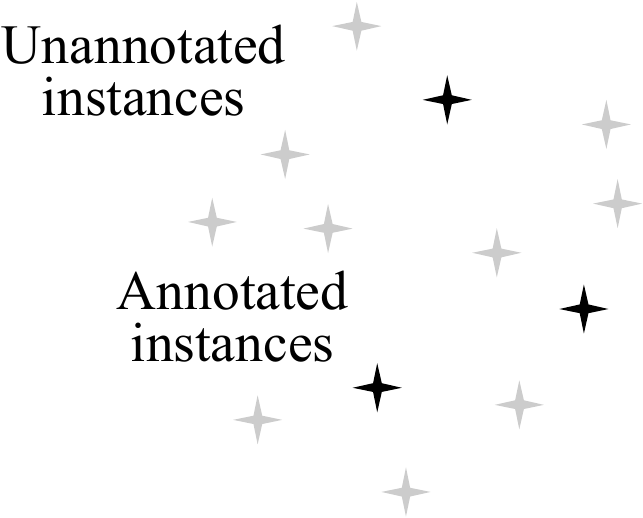}}
    \end{minipage}
    \hspace{5pt}
    \begin{minipage}{0.23\linewidth}
    \centering
    \subfigure[]{
    \centering
    \includegraphics[width=1.0\textwidth]{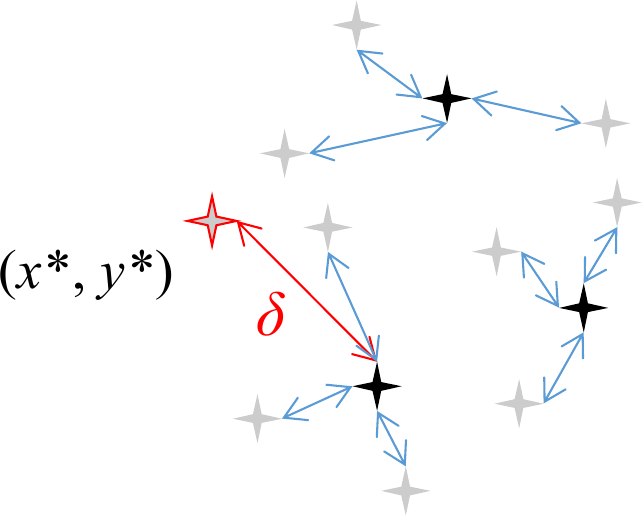}}
    \end{minipage}
   
    \caption{\small \label{fig:sys_config} (a) The APOD pipeline alternates between POD and AIS.
    (b) Individual selection: The annotated and unannotated instances from subgroup $\mathscr{U}_{\tilde{a}}^{\tilde{c}}$, where $\tilde{a}=0$ and $\tilde{c}=1$.
    (c) Each unannotated instance is connected to an annotated instance determined by $\min_{\boldsymbol{x}_j \in \mathcal{S}} || \boldsymbol{h}_i - \boldsymbol{h}_j ||_2$ (marked as blue arrows). The red \textcolor{red}{$\delta$} denotes the largest distance pair which selects the best candidate for annotation.  
    }
    
\end{figure}

\section{\Algnameb{}}

In this section, we introduce the \Algnamea{} (\Algnameabbr{}) framework to mitigate algorithmic bias under a limited annotation budget.
As shown in Figure \ref{fig:sys_config}~(a), \Algnameabbr{} integrates Penalization Of Discrimination~(POD) and Active Instance Selection~(AIS) in a unified and iterative framework.
Specifically, in each iteration, POD focuses on debiasing the classifier head $f_h$ on the partially annotated dataset $\{ (\boldsymbol{x}_i, y_i, a_i) \in \mathcal{S} \}$ and $\{ (\boldsymbol{x}_i, y_i) \in \mathscr{U} \}$, 
while AIS selects the optimal instance $(\boldsymbol{x}^*, y^*)$ from the unannotated dataset $\mathscr{U}$ that can further promote bias mitigation.
Sensitive attributes of the selected instances will be annotated by human experts: $(\boldsymbol{x}^*, y^*)  \to  (\boldsymbol{x}^*, y^*, a^*)$.
After that, these instances will be moved from the unannotated dataset $\mathscr{U}\leftarrow\mathscr{U}\setminus \{ (\boldsymbol{x}^*, y^*) \}$ to the annotated dataset $\mathcal{S} \leftarrow \mathcal{S} \cup \{ (\boldsymbol{x}^*, y^*, a^*) \}$ for debiasing the classifier in the next iteration.
The POD and AIS are introduced as follows. 


\subsection{Penalization Of Discrimination~(POD)}
\label{sec:PD}
 
POD learns a fair classifier head $f_h$ via bias penalization on both annotated instances $\{(\boldsymbol{x}_i, y_i, a_i) \in \mathcal{S}\}$ and unannotated instances $\{(\boldsymbol{x}_i, y_i) \in \mathscr{U}\}$.
To be concrete, POD considers a regularization term, consisting of the true and false positive rate difference\footnote{The combination of TPR and FPR is representative enough accross different fairness metrics. POD is flexible to use other metrics as the regularizer for the bias mitigation.}, to balance the model performance on different subgroups.
In this way, given $\boldsymbol{h}_i = f_b(\boldsymbol{x}_i | \theta_b)$, $f_h$ is updated to minimize the hybrid loss function given by
\begin{equation}
\begin{aligned}
\label{eq:non_diff_loss}
L = \sum_{i=1}^N l(\boldsymbol{h}_i, y_i; \theta_h) + \lambda ( \Delta \text{TPR}^2 + \Delta \text{FPR}^2 ),
\end{aligned}    
\end{equation}
where $l(\boldsymbol{h}_i, y_i; \theta_h)$ denotes the cross-entropy loss, and the term $\Delta \text{TPR}^2 + \Delta \text{FPR}^2$ penalizes the bias in $f_h$ to improve fairness, controlled by the hyper-parameter~$\lambda$.

However, Equation~(\ref{eq:non_diff_loss}) is not feasible to debias $f_h$ in an end-to-end manner, since neither $\text{TPR}$ nor $\text{FPR}$ is differentiable with respect to the parameters $\theta_h$.
It is thus necessary to reformulate $\Delta \text{TPR}$ and $\Delta \text{FPR}$, which involves the parameterization of true and false positive rate with respect to $\theta_h$, respectively.
For notation convenience and without the loss of generality, we unify the formulation of true and false positive rates by
\begin{equation}
\begin{aligned}
    \label{eq:false_prediction_rete}
    p_a(y, c) &= \mathbb{P} (\hat{Y}=c \mid Y=y, A = a),
\end{aligned}
\end{equation}
where we take $y=1,c=1$ to have $p_a(1,  1) = \text{TPR}_{A=a}$ and $y=0,c=1$ to have $p_a(0,  1) = \text{FPR}_{A=a}$.
To parameterize $p_a(y,  c)$ with respect to $\theta_h$, we reformulate it as follows
\begin{align}
\label{eq:parameterize_FPR_2}
p_a(y, c) &= \frac{\sum_{(\boldsymbol{x}_i, y_i, a_i) \in \mathcal{S}_a^y} \boldsymbol{1}_{\hat{y}_i = c}}{|\mathcal{S}_a^y|}
= \frac{\sum\limits_{(\boldsymbol{x}_i, y_i, a_i) \in \mathcal{S}_a^y}  \text{sgn}( f^{c}_h( \boldsymbol{h}_i ) - f^{1-c}_h( \boldsymbol{h}_i ) ) }{|\mathcal{S}_a^y|}
\\
\label{eq:parameterize_FPR_3}
&\approx \frac{ \sum\limits_{(\boldsymbol{x}_i, y_i, a_i) \in \mathcal{S}_a^y}  \lambda ( f^{c}_h( \boldsymbol{h}_i ) - f^{1-c}_h( \boldsymbol{h}_i ) ) }{|\mathcal{S}_a^y|}
\triangleq \lambda \tilde{p}_a(y, c),
\end{align}
where $\text{sgn}(x)=0$ for $x<0$ and $\text{sgn}(x)=1$ for $x\geq0$. 
Here we relax $\text{sgn}(x)$ with a linear function\footnote{It also has other choices for the relaxation, e.g. sigmoid and tanh functions. The linear function is chosen for simplicity.} $\lambda x$ in the approximation of Equation~(\ref{eq:parameterize_FPR_2}) to make $p_a(y, c)$ differentiable with respect to $\theta_h$; $\mathcal{S}_a^y = \{ (\boldsymbol{x}_i, y_i, a_i) \in \mathcal{S} \mid a_i = a, y_i = y \}$ for $a \!\in\! \mathcal{A}, y \!\in\! \mathcal{Y}$; and $f^i_h(\boldsymbol{h})$ denotes element $i$ of $f_h(\boldsymbol{h})$ for $i \!\in\! \mathcal{Y}$. 
Based on the relaxed regularization term, $f_h$ is updated to minimize the loss function given by
\begin{equation}
\begin{aligned}
\label{eq:APD_loss_function}
L = \frac{1}{N}  \sum_{i=1}^N l(\boldsymbol{h}_i,  y_i; \theta_h) + \lambda  \sum_{y \in \mathcal{Y}}  
\big[ \tilde{p}_0(y,  1) - \tilde{p}_1(y,  1) \big]^2 ,
\end{aligned}    
\end{equation}
where the estimation of cross-entropy  $\frac{1}{N} \sum_{i=1}^N l(\boldsymbol{h}_i, y_i; \theta_h)$ includes both annotated and unannotated instances; the regularization term $[ \tilde{p}_0(y, 1) - \tilde{p}_1(y, 1) ]^2$ is calculated using the annotated instances; and the hyper-parameter $\lambda$ controls the importance of regularization.

\subsection{Active Instance Selection~(AIS)}
\label{sec:dis}

In each iteration, AIS selects instances from the unannotated dataset $\mathscr{U}$ to annotate the sensitive attribute values.
The newly annotated instances are merged with the dataset for debiasing the classifier head in subsequent iterations.
The AIS process consists of two steps: (1) \textbf{Group selection} is to select the subgroup $\mathscr{U}_{\tilde{a}}^{\tilde{c}} = \{ (\boldsymbol{x}_i, y_i) \in \mathscr{U} \mid a_i = \tilde{a}, y_i = \tilde{c} \}$ on which the model has the worst performance; (2) \textbf{Individual selection} is to select the optimal instance from $\mathscr{U}_{\tilde{a}}^{\tilde{c}}$, which can mostly expose the existing bias
of the model for promoting the bias mitigation in the next iteration.

\paragraph{\textbf{Group Selection}} is motivated by the observation that adding more instances to the subgroup having the worst classification accuracy can improve the fairness by increasing its contribution to the average loss~\cite{hashimoto2018fairness,lahoti2020fairness}.
Specifically, for group selection, the unannotated dataset $\mathscr{U}$ is splitted into $\{\mathscr{U}_{a}^{c}\}_{a \in \mathcal{A}, c \in \mathcal{Y}}$, where $\mathscr{U}_{a}^{c} = \{ (\boldsymbol{x}_i, y_i) \in \mathscr{U} | a_i \!=\! a, y_i \!=\! c \}$ denotes a subgroup of unannotated instances. 
We estimate the classification accuracy $p_a(c, c)=\mathbb{P}(\hat{Y}=c|A=a,Y=c)$ to evaluate $f$ on each subgroup $\mathscr{U}_{a}^{c}$ for $a \in \mathcal{A}$ and $c \in \mathcal{Y}$, respectively, following Equation (\ref{eq:false_prediction_rete}).
In this way, the subgroup $\mathscr{U}_{\tilde{a}}^{\tilde{c}} = \{ (\boldsymbol{x}_i, y_i) \!\in\! \mathscr{U} | a_i = \tilde{a}, y_i = \tilde{c} \}$ which suffers from the worst accuracy is selected by
\begin{equation}
\label{eq:group_selection}
\begin{aligned}
\tilde{a}, \tilde{c} = \mathop{\arg\min}_{a \in \mathcal{A}, c \in \mathcal{Y}} p^*_a(c, c),
\end{aligned}
\end{equation}
where $p^*_a(c, c) = p_a(c, c) - (p_0(c, c) + p_1(c, c))/2$ denotes the centralized classification accuracy after considering the performance divergence of the classifier on different classes. 
For example, in Figure~\ref{fig:unfair_vs_rs_vs_APD}~(b), we select the subgroup with the worst accuracy $\mathscr{U}_{0}^{1}$ which corresponds to the positive instances from group 0, due to the fact that $p^*_0(1,1) < p^*_0(0,0), p^*_1(0,0), p^*_1(1,1)$.

Note that $p^*_a(c, c)$ cannot be estimated without the annotations of sensitive attribute. 
We thus learn another classifier head $f_a: \mathbb{R}^M \!\to\! \mathbb{R}^{|\mathscr{A}|}$ to predict the sensitive attribute $\hat{a} \!=\! \mathop{\arg\max} f_a(\boldsymbol{h}_i | \theta_a)$ for the unannotated instances $\boldsymbol{x}_i \!\in\! \mathscr{U}$, where $f_a$ is updated on the annotated set $\mathcal{S}$ by minimizing the cross-entropy loss
\begin{equation}
    \theta^*_a = \frac{1}{|\mathcal{S}|} \sum_{(\boldsymbol{x}_i, y_i, a_i) \in \mathcal{S}} l(\boldsymbol{h}_i, a_i; \theta_a).
\end{equation}

\paragraph{\textbf{Individual Selection}} aims to proactively select the most representative instances for annotation, which can maximally promote bias mitigation. Since the classifier $f$ has the worst accuracy on subgroup $\mathscr{U}_{\tilde{a}}^{\tilde{c}}$, reducing the classification error on $\mathscr{U}_{\tilde{a}}^{\tilde{c}}$ would improve fairness, where $\tilde{a}$ and $\tilde{c}$ are chosen through group selection in Equation~(\ref{eq:group_selection}).
The strategy of individual selection is to expand the annotated dataset to reduce $\delta$-cover of subgroup $\mathscr{U}_{\tilde{a}}^{\tilde{c}}$~\cite{sener2018active}.
Specifically, the annotated dataset $\mathcal{S}$ enables $\delta$-cover of the entire dataset $\mathscr{D}$ if $\forall \boldsymbol{x}_i \in \mathscr{D}$, $\exists \boldsymbol{x}_j \in \mathcal{S}$ such that $|| \boldsymbol{x}_i - \boldsymbol{x}_j ||_2 \leq \delta$, where $\delta$ denotes the coverage radius given by 
\begin{equation}
\begin{aligned}
\label{eq:delta_cover}
\delta = \max_{\boldsymbol{x}_i \in \mathscr{D}} \min_{\boldsymbol{x}_j \in \mathcal{S}} || \boldsymbol{x}_i - \boldsymbol{x}_j ||_2.
\end{aligned}
\end{equation}
Furthermore, it is observed that the generalization error of a model approaches the training error\footnote{The training error is less than generalization error in most cases.} if the coverage radius $\delta$ is small~\cite{sener2018active}.
Following such scheme, we select the instance in subgroup $\mathscr{U}_{\tilde{a}}^{\tilde{c}}$, which could decrease $\delta$-coverage to reduce the classification error on $\mathscr{U}_{\tilde{a}}^{\tilde{c}}$. 
To be concrete, the distance between $\boldsymbol{x}_i$ and $\boldsymbol{x}_j$ is measured by $|| \boldsymbol{h}_i - \boldsymbol{h}_j ||_2$, where $\boldsymbol{h}_i=f_b(\boldsymbol{x}_i | \theta_b)$ and $\boldsymbol{h}_j=f_b(\boldsymbol{x}_j | \theta_b)$ are the embeddings of $\boldsymbol{x}_i$ and $\boldsymbol{x}_i$, respectively. 
We have the instance $(\boldsymbol{x}^*, y^*)$ selected for annotation following the max-min rule
\begin{equation}
\label{eq:individual_selection}
    (\boldsymbol{x}^*, y^*) = \mathop{\arg\max}_{(\boldsymbol{x}_i, y_i) \in \mathscr{U}_{\tilde{a}}^{\tilde{c}}} \min_{(\boldsymbol{x}_j, y_j) \in \mathcal{S}}  || \boldsymbol{h}_i - \boldsymbol{h}_j ||_2.
\end{equation}
The individual selection strategy is illustrated in Figures~\ref{fig:sys_config}~(b) and (c), where $\delta$ reduction guides the individual selection. The candidate instances in $\mathscr{U}_{\tilde{a}}^{\tilde{c}}$ and annotated instances are shown in Figure~\ref{fig:sys_config}~(b). The distances between each candidate instance and annotated instances are measured in embedding space $|| \boldsymbol{h}_i - \boldsymbol{h}_j ||_2$, where the minimal one is marked as a blue arrow. 
The red instance marked by $(x^*, y^*)$ in Figure~\ref{fig:sys_config}~(c) indicates the best candidate to be annotated.

\begin{multicols}{2}
\begin{algorithm}[H] 
\small
\caption{\text{\Algnameabbr{}{}}.}
\label{alg:APD}
\textbf{Input}:{\small initial annotated dataset $\mathcal{S}$} \\
\textbf{Output}:{\small classifier body $f_b$, head $f_h$}\!\!\!\!

{\small $\theta_b^*, \! \theta_h^* \!=\! \mathop{\arg\min} \sum_{i=1}^N \!\! l(\boldsymbol{x}_i, y_i; \theta_b, \theta_h)$}\!\!\!\!

\While{\textit{within budget limit}}{

\textcolor{gray}{\# Penalization Of Discrimination.}

$\theta^*_h = $ POD($\mathcal{S}$, $f_b$, $f_h$)

\textcolor{gray}{\# Active Instance Selection.}

$(\boldsymbol{x}^*, y^*) = \text{AIS}(f_b, f_h)$

$\mathcal{S} = \mathcal{S} \cup \{ (\boldsymbol{x}^*, y^*, a^*) \}$ and $\mathscr{U} = \mathscr{U} \setminus \{ (\boldsymbol{x}^*, y^*) \}$

}

\end{algorithm}

\begin{algorithm}[H] 
\setstretch{1.0}
\small
\caption{POD}
\label{alg:PD}
\textbf{Input}: annotated dataset $\mathcal{S}$, classifier body $f_b$, head $f_h$. \\
\textbf{Output}: fair classifier head $f_h^*$.


\While{\textit{not converged}}{

For $a \in \mathcal{A}$ and $y \in \mathcal{Y}$, estimate $\tilde{p}_a(y,1)$ given by Equation~(\ref{eq:parameterize_FPR_3}).

Update the classifier head $f_h$ to minimize the loss function in Equation~(\ref{eq:APD_loss_function}).

}
\Return $f_h^*$
\end{algorithm}

\end{multicols}

\subsection{The \Algnameabbr{} Algorithm}

The details of \Algnameabbr{} are summarized in Algorithm~\ref{alg:APD}.
Initially, \Algnameabbr{} learns the biased $f_b$ and $f_h$, and randomly samples a small set of annotated instances $\mathcal{S}$.
In each iteration, \Algnameabbr{} 
first learns $f_a$ to predict the sensitive attribute of unannotated instances;
then debiases $f_h$ via POD~(line~5);
after this, \Algnameabbr{} selects the optimal instance $(\boldsymbol{x}^*, y^*)$ for annotation via AIS~(line~6) and merges the selected instance with the annotated dataset~(line~7);
POD and AIS are given in Algorithms~\ref{alg:PD} and~\ref{alg:AIS}, respectively;
the iteration stops once the number of  annotated instance reaches the budget.


\subsection{Theoretical Analysis}

We theoretically investigate the proposed \Algnameabbr{} to guarantee that bias mitigation is globally achieved, as shown in Theorem~\ref{pp:upper_bound}. We then demonstrate the effectiveness of AIS~(both group selection and individual selection) in Remark~\ref{sec:rk1}.
The proof of Theorem~\ref{pp:upper_bound} is given in Appendix~\ref{sec:proof_theorem_appendix}.
\begin{theorem}
\label{pp:upper_bound}
Assume the loss value on the training set has an upper bound $\frac{1}{|\mathcal{S}|} \sum_{(\boldsymbol{x}_i, y_i, a_i) \in \mathcal{S}} l(\boldsymbol{h}_i, y_i;  \theta_h) \leq \epsilon$\footnote{\scriptsize $\epsilon$ can be very small if the classifier head $f_h$ has been well-trained on the annotated dataset $\mathcal{S}$.}, 
and $l(\boldsymbol{h}, y; \theta_h)$ and $f_h$ satisfy $K_l$- and $K_h$-Lipschitz continuity\footnote{\scriptsize $l(\boldsymbol{h}, y;  \theta_h)$ and $f_h$ satisfy $|l(\boldsymbol{h}_i, y; \theta_h) - l(\boldsymbol{h}_j, y; \theta_h)| \leq K_l ||\boldsymbol{h}_{i} - \boldsymbol{h}_{j}||_2$ and $|p(y | \boldsymbol{x}_i) - p(y | \boldsymbol{x}_j)| \leq K_h ||\boldsymbol{h}_{i} - \boldsymbol{h}_{j}||_2$, respectively, where the likelihood function $p(y \mid \boldsymbol{x}_i) = \text{softmax}(f_h(\boldsymbol{h}_i | \theta_h))$.}, respectively.
The generalization loss difference between the unprivileged group and the privileged group has the following upper bound with probability $1-\gamma$, 
\begin{align}
&\bigg| \! \int_{\mathcal{X}_0} \!\! \int_{\mathcal{Y}} \! p(\boldsymbol{x},  y) l(\boldsymbol{h},  y;  \theta_h) \mathrm{d}\boldsymbol{x} \mathrm{d}y \!-\!\! \int_{\mathcal{X}_1} \!\! \int_{\mathcal{Y}} \! p(\boldsymbol{x},  y) l(\boldsymbol{h},  y;  \theta_h) \mathrm{d}\boldsymbol{x} \mathrm{d}y \bigg|
\nonumber
\\
\label{eq:generation_upper_bound2}
&\quad\quad\quad\quad\quad\quad\quad\quad \leq \epsilon + \min \Big\{ \sqrt{-L^2 \log \gamma (2 N_{\tilde{a}})^{-1}},  (K_l + K_h L) \delta_{\tilde{a}} \Big\}, 
\end{align}
where $\tilde{a} \!=\! \mathop{\arg\max}_{a \in \mathcal{A}}  \int_{\mathcal{X}_a}   \int_{\mathcal{Y}}  p(\boldsymbol{x}, y) l(\boldsymbol{h}, y; \theta_h) \mathrm{d}\boldsymbol{x} \mathrm{d}y$; $\mathcal{X}_a \!=\! \{\boldsymbol{x}_i \in \mathscr{D} | a_i \!=\! a \}$;
$\delta_{\tilde{a}} \!=\! \max_{\boldsymbol{x}_i \in \mathcal{X}_{\tilde{a}}} \!\! \min_{(\boldsymbol{x}_j, y_j, a_j) \in \mathcal{S}} \!\! || \boldsymbol{h}_i \!-\! \boldsymbol{h}_j ||_2$;  $N_{\tilde{a}} \!=\! |\{ (\boldsymbol{x}_i, y_i, a_i) | a_i \!=\! \tilde{a}, (\boldsymbol{x}_i, y_i, a_i) \!\in\! \mathcal{S} \}|$;
$L \!=\! \max_{(\boldsymbol{x}_i, y_i) \in \mathscr{U}} l(\boldsymbol{h}_i, y_i; \theta_h)$; 
and $\boldsymbol{h}_i = f_b(\boldsymbol{x}_i | \theta_b)$.



\end{theorem}

In Theorem~\ref{pp:upper_bound}, the global fairness is formalized via considering the generalization error difference between the unprivileged and privileged group as the relaxed fairness metric, and  \Algnameabbr{} contributes to the global fairness via explicitly tightening the upper bound of the relaxed fairness metric.
We demonstrate the details that AIS can iteratively tighten the bound in Remark~\ref{sec:rk1}.

\begin{algorithm}[t] 
\setstretch{0.95}
\small
\caption{\small Active Instance Selection (AIS).}
\label{alg:AIS}
\textbf{Input}: classifier body $f_b$ and classifier head $f_h$. \\
\textbf{Output}: the selected instance $(\boldsymbol{x}^*, y^*)$.

Update $f_a$ to minimize $\frac{1}{|\mathcal{S}|} \sum_{(\boldsymbol{x}_i, y_i, a_i) \in \mathcal{S}} l(\boldsymbol{h}_i, a_i; \theta_a)$.

Estimate the sensitive attribute $\hat{a}_i = \mathop{\arg\max} f_a(\boldsymbol{h}_i \mid \theta_a)$ for $\boldsymbol{x}_i \in \mathscr{U}$.


For $a \in \mathcal{A}$ and $c \in \mathcal{Y}$, estimate the classification accuracy $p_a(c, c)=\mathbb{P}(\hat{Y}=c|\hat{A}=a,Y=c)$ on subgroup $\mathscr{U}_{a}^{c}$.

For $a \!\in\! \mathcal{A}$ and $c \!\in\! \mathcal{Y}$, centralize $p_a(c, c)$ into $p^*_a(c, c)$ by $p^*_a(c, c) = p_a(c, c) - \frac{p_0(c, c) + p_1(c, c)}{2}$.

Execute the group selection by $\tilde{a}, \tilde{c} = \mathop{\arg\min}_{a\in \mathcal{A}, c\in \mathcal{Y}} p^*_a(c, c)$.
%

Execute the individual selection by
\begin{equation}
\setlength\abovedisplayskip{0mm}
\setlength\belowdisplayskip{0mm}
    (\boldsymbol{x}^*, y^*) = \mathop{\arg\max}_{(\boldsymbol{x}_i, y_i) \in \mathscr{U}_{\tilde{a}}^{\tilde{c}}} \min_{(\boldsymbol{x}_j, y_j, a_j) \in \mathcal{S}} || \boldsymbol{h}_i - \boldsymbol{h}_j ||_2.
    \nonumber
\end{equation}
 
\end{algorithm}

\begin{remark}
\label{sec:rk1}
In each iteration of \Algnameabbr{}, the group selection reduces the value of $\sqrt{-L^2 \log \gamma (2 N_{\tilde{a}})^{-1}}$ by merging a new instance $(\boldsymbol{x}_i, y_i, a_i)|_{a_i=\tilde{a}}$ to the annotated dataset $\mathcal{S}$ to increase the value of $N_{\tilde{a}} \!=\! |\{ (\boldsymbol{x}_i, y_i, a_i) \!\in\! \mathcal{S} | a_i \!=\! \tilde{a} \}|$.
Here, we adopt an approximation given by Equation~(\ref{eq:group_approximation}) due to the negative relationship between the accuracy and the generalization loss,
\begin{equation}
\label{eq:group_approximation}
\begin{aligned}
    \tilde{a} &= \mathop{\arg\min}_{a \in \mathcal{A}} p_a^*(c, c) \approx \mathop{\arg\max}_{a \in \mathcal{A}} \int_{\mathcal{X}_a} \int_{\mathcal{Y}_c} p(\boldsymbol{x}, y) l(\boldsymbol{h}, y; \theta_h) \mathrm{d}\boldsymbol{x} \mathrm{d}y,
\end{aligned}
\end{equation}
where $\mathcal{Y}_c \!=\! \{ y \!=\! c ~|~ y \!\in\! \mathcal{Y} \}$ for $c \!\in\! \mathcal{Y}$. 
Meanwhile, the individual selection reduces the value of $\delta_{\tilde{a}}$ by selecting an instance following Equation~(\ref{eq:individual_selection}).
With the combination of group selection and individual selection, \Algnameabbr{}  contributes to the decline of $\min \{ \sqrt{-L^2 \log \gamma (2 N_{\tilde{a}})^{-1}}, (K_l + K_h L) \delta_{\tilde{a}} \}$, which leads to
tightening the upper bound of the fairness metric in Equation~(\ref{eq:generation_upper_bound2}). 
\end{remark}

Remark~\ref{sec:rk1} reveals that both group selection and individual selection of the two-step AIS are effective in tightening the upper bound of relaxed fairness metric.
Compared to AIS, we consider two compositional instance selection methods: one with group selection alone, where we randomly select an instance $(\boldsymbol{x}^*, y^*)$ from the subgroup
$\mathscr{U}_{\tilde{a}}^{\tilde{c}}$ satisfying $\tilde{a}, \tilde{c} \!=\! \mathop{\arg\min}_{a\in \mathcal{A}, c\in \mathcal{Y}} p_a^*(c, c)$;
and another with individual selection alone, where an instance is selected via
$(\boldsymbol{x}^*, y^*) \!=\! \mathop{\arg\max}_{(\boldsymbol{x}_i, y_i) \in \mathscr{U}} \min_{(\boldsymbol{x}_j, y_j, a_j) \in \mathcal{S}} || \boldsymbol{h}_i \!-\! \boldsymbol{h}_j ||_2$ without the selection of subgroup.
According to Remark~\ref{sec:rk1}, the compositional methods merely enable to reduce one of the terms $(2 N_{\tilde{a}})^{-1}$ or $\delta_{\tilde{a}}$ in Equation~(\ref{eq:generation_upper_bound2}), which are less effective than the two-step AIS as an unit.

\section{Experiment}

In this section, we conduct experiments to evaluate \Algnameabbr{}, aiming to answer the following research questions: 
\textbf{RQ1}: In terms of comparison with state-of-the-art baseline methods, does \Algnameabbr{} achieve more effective mitigation of unfairness under the same annotation budget?
\textbf{RQ2}: Does \Algnameabbr{} select more informative annotations for bias mitigation than baseline methods?
\textbf{RQ3}: How does the ratio of annotated instances affect the mitigation performance of \Algnameabbr{}?
\textbf{RQ4}: Do both group selection and individual selection in the AIS contribute to bias mitigation?
The experiment settings including the datasets and implementation details are given in Appendix~\ref{sec:dataset_appendix} and \ref{sec:implement_appendix}, respectively.

\subsection{Bias Mitigation Performance Analysis (RQ1)}

In this section, we compare our proposed \Algnameabbr{} with three state-of-the-art baseline methods of bias mitigation.
The key component of the baseline methods are given as follows.
\textbf{Vanilla}: The classifier is trained without bias mitigation. 
\noindent
\textbf{Group DRO}~\cite{sagawa2019distributionally}: Group DRO utilizes all sensitive information to minimize the classification loss on the unprivileged group to reduce the performance gap between different sensitive groups.
\noindent 
\textbf{Learning from Failure~(LfF)}~\cite{nam2020learning}: As a debiasing method that relies on proxy sensitive annotations, LfF adopts generalized cross-entropy loss to learn a proxy annotation generator, and proposes a re-weighted cross entropy loss to train the debiased model.
\noindent 
\textbf{Fair Active Learning~(FAL)}~\cite{anahideh2020fair}: The instance selection in FAL is to maintain a subset of annotated instances for model training, which is not guided by gradient-based model debiasing.
More details are given in the Appendix~\ref{sec:baseline_appendix}. 


\begin{figure*}[t]
\setlength{\abovecaptionskip}{1mm}
\setlength{\belowcaptionskip}{-2mm}
\centering
\begin{minipage}{0.32\linewidth}
\centering
\subfigure[MEPS.]{
\centering
\includegraphics[width=1.0\textwidth]{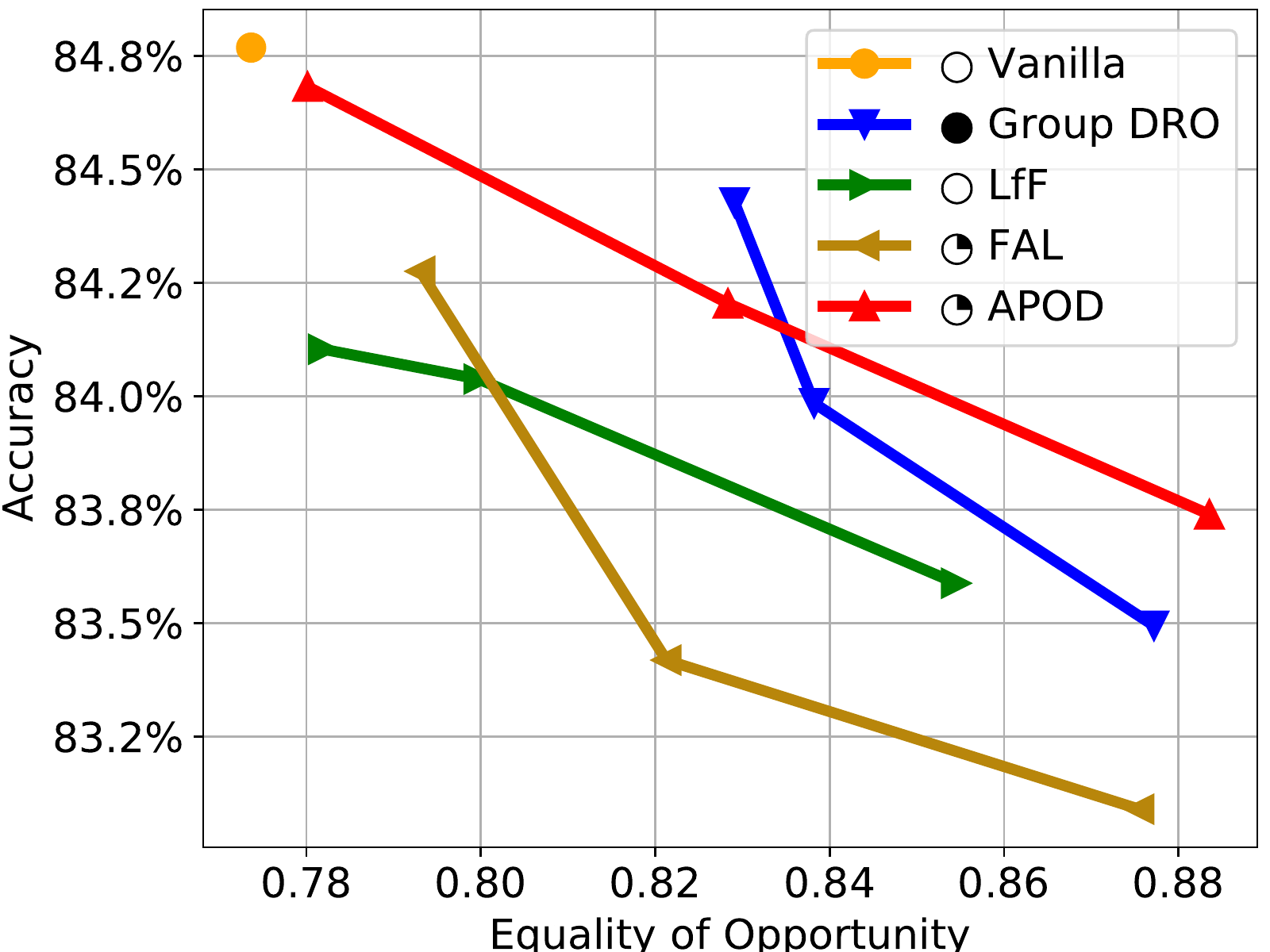}}
\end{minipage}
\begin{minipage}{0.32\linewidth}
\centering
\subfigure[German credit.]{
\centering
\includegraphics[width=1.0\textwidth]{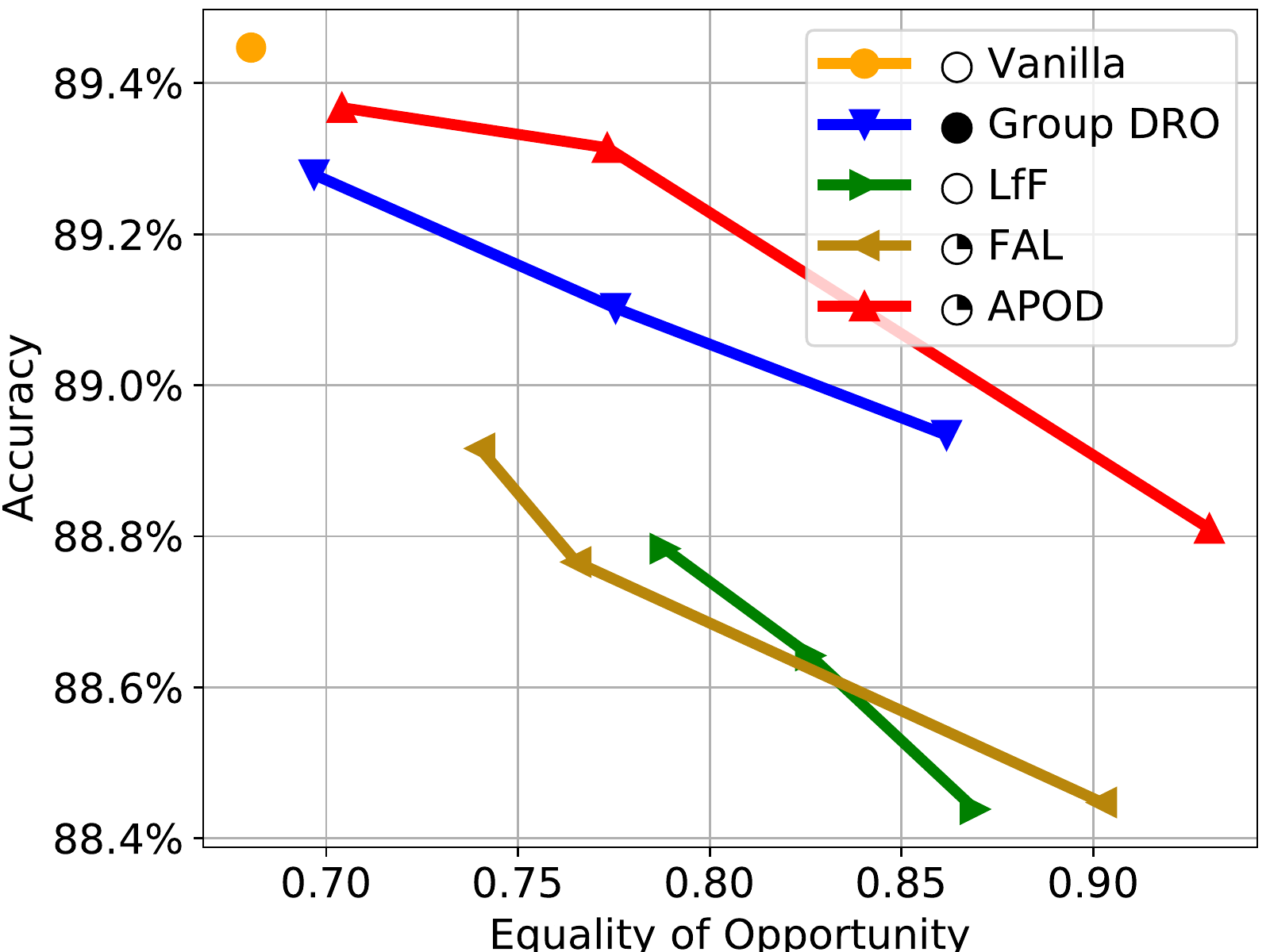}}
\end{minipage}
\begin{minipage}{0.32\linewidth}
\centering
\subfigure[Loan default.]{
\centering
\includegraphics[width=1.0\textwidth]{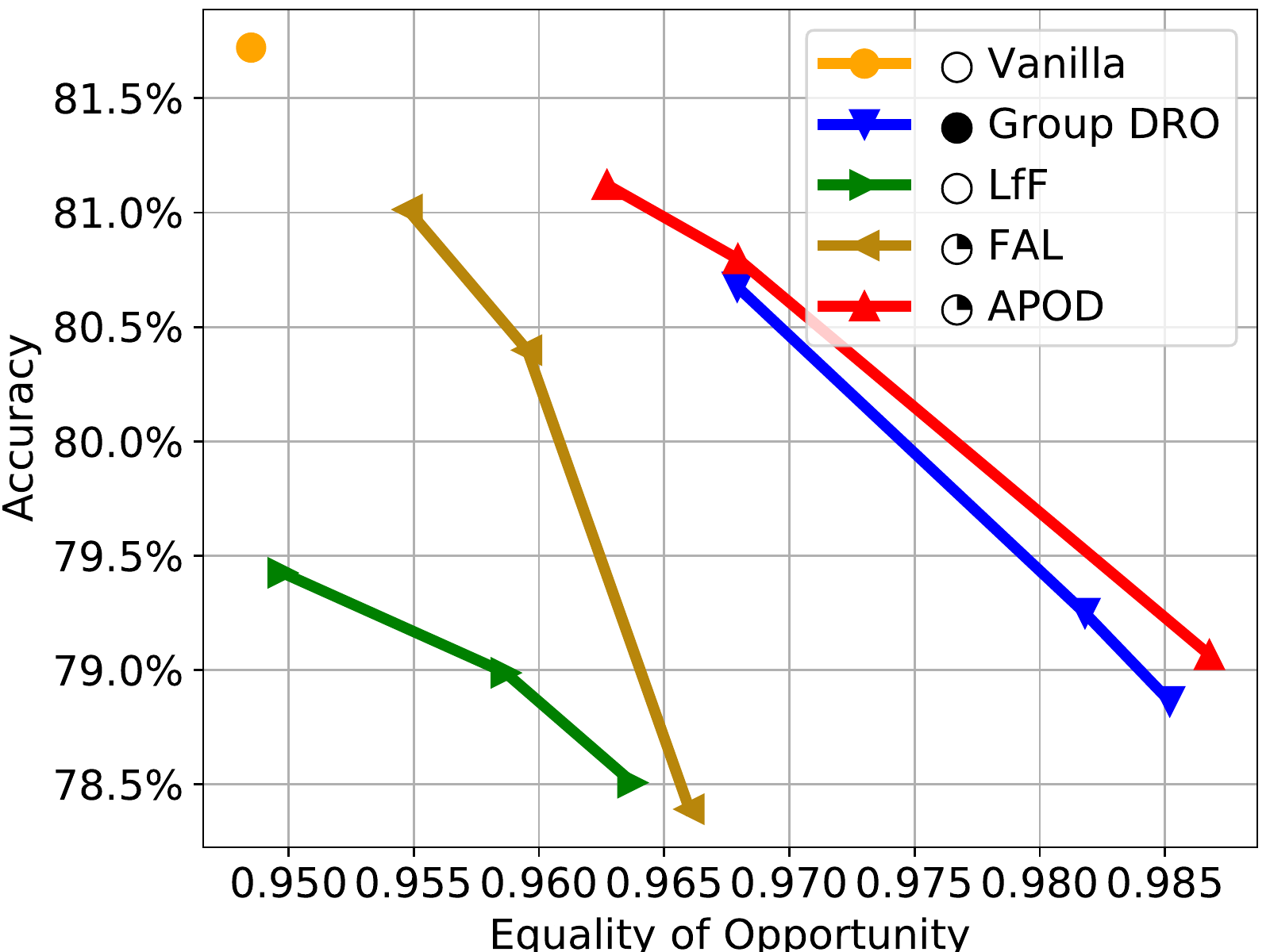}}
\end{minipage}
\begin{minipage}{0.32\linewidth}
\centering
\subfigure[Adult.]{
\centering
\includegraphics[width=1.0\textwidth]{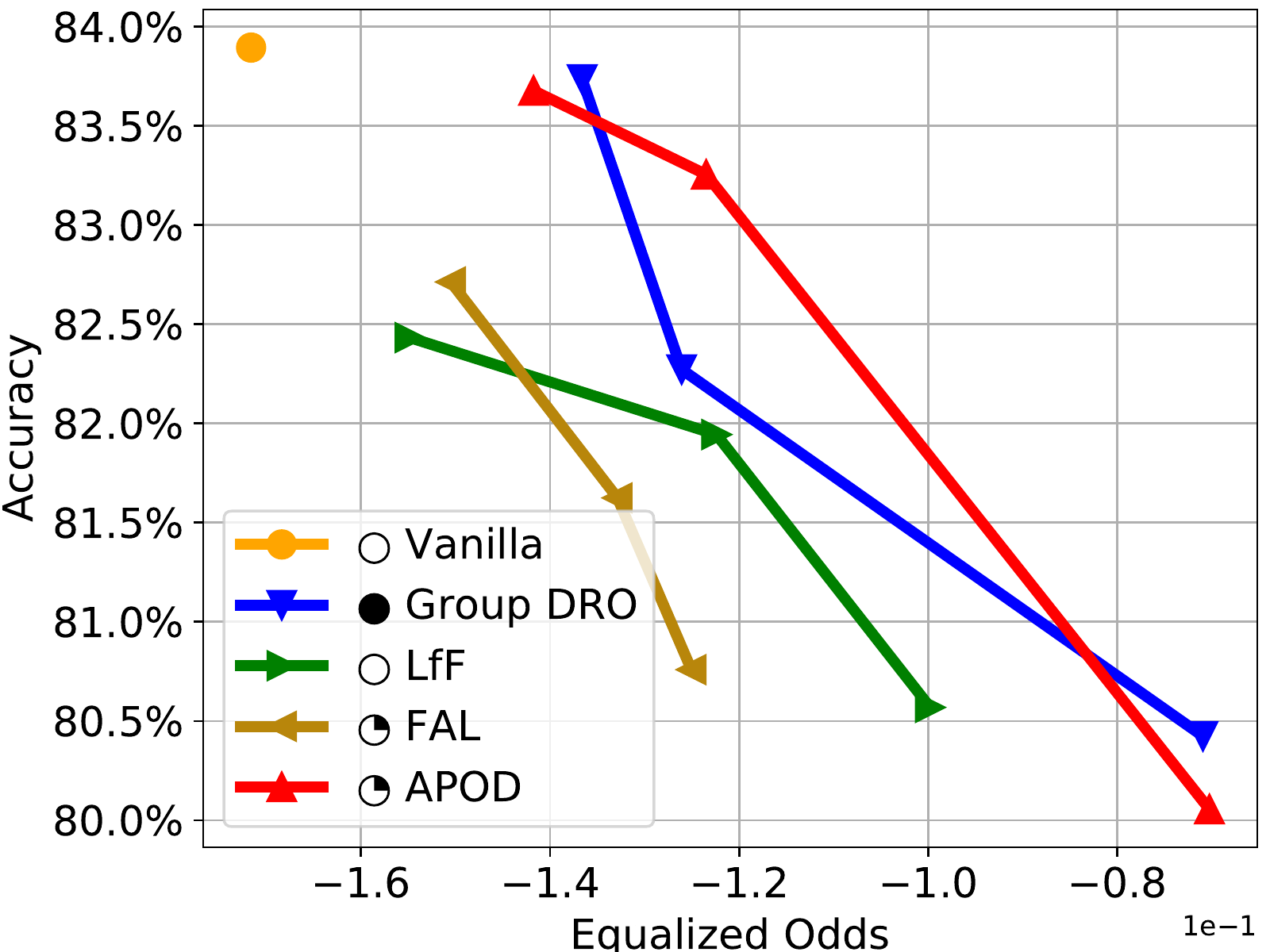}}
\end{minipage}
\begin{minipage}{0.32\linewidth}
\centering
\subfigure[CelebA-wavy hair.]{
\centering
\includegraphics[width=1.0\textwidth]{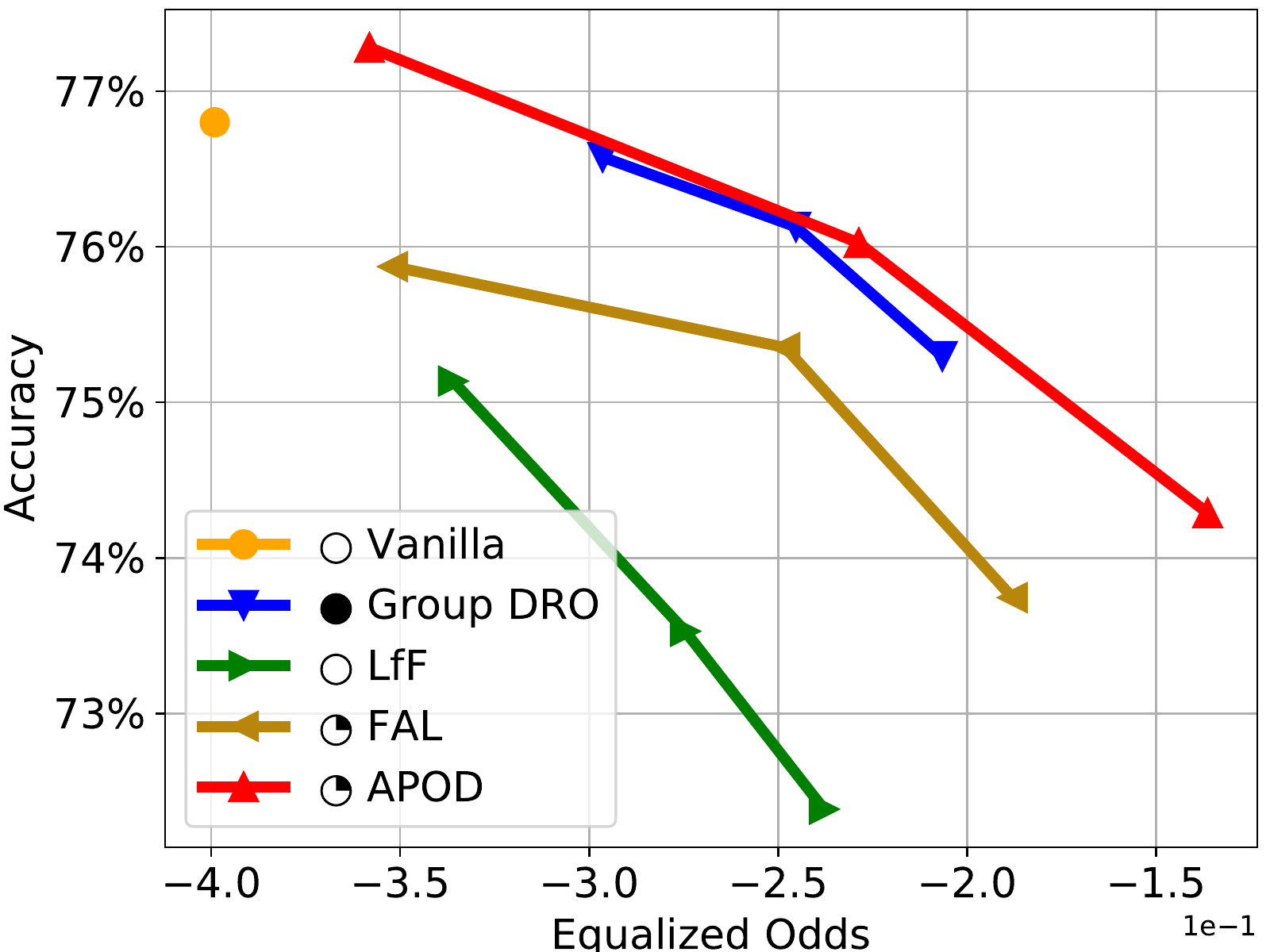}}
\end{minipage}
\begin{minipage}{0.32\linewidth}
\centering
\subfigure[CelebA-young.]{ 
\centering
\includegraphics[width=1.0\textwidth]{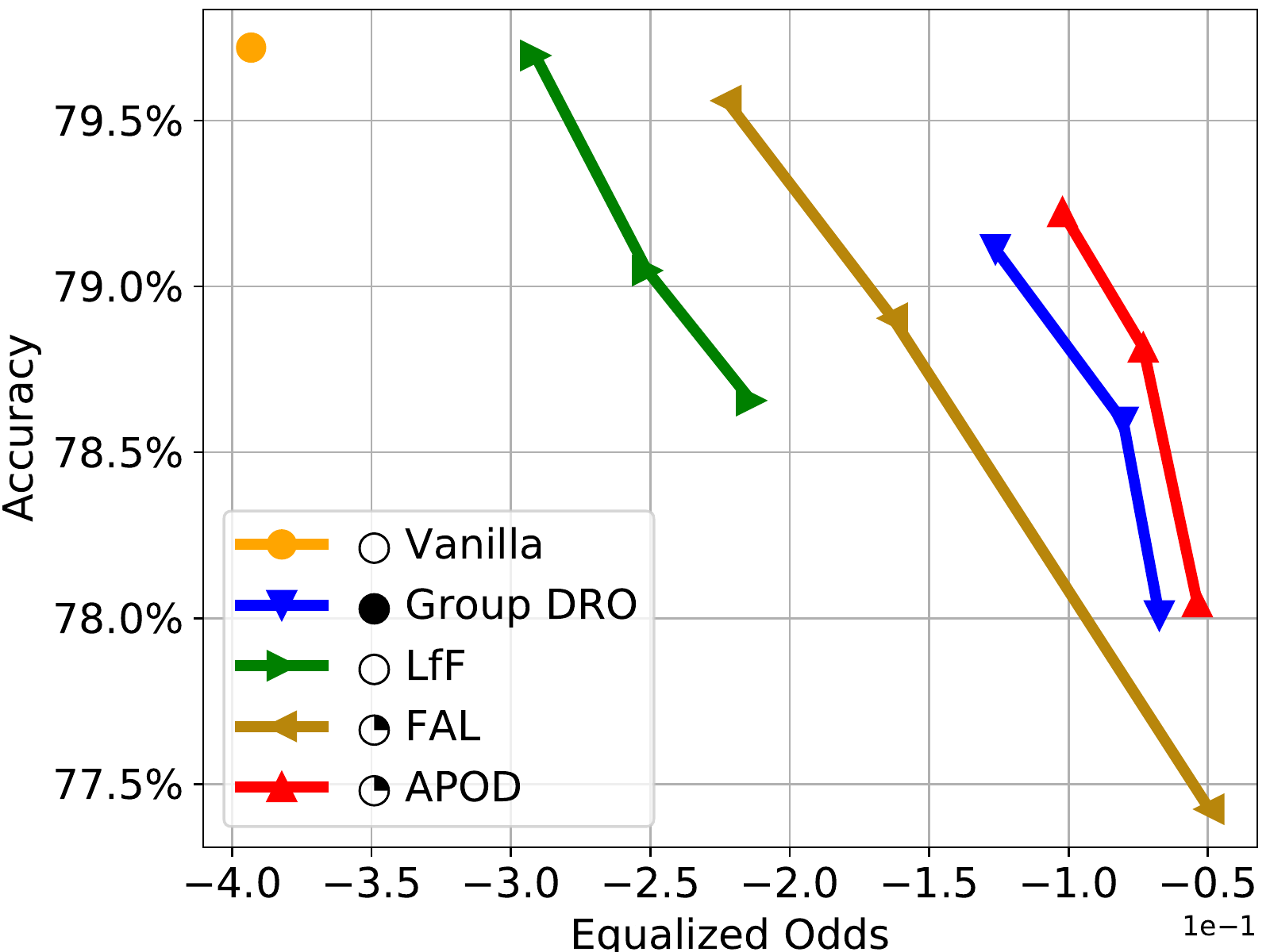}}
\end{minipage}
\caption{Accuracy-fairness curve; Algorithm: Vanilla training, Group DRO, LfF, FAL and \Algnameabbr{}; Dataset: 
(a) MEPS, 
(b) German credit, 
(c) Loan default,
(d) Adult,
(e) CelebA-wavy hair,
(f) CelebA-young.
}
\label{fig:exp_sota_results}
\end{figure*}

To have a fair comparison, we unify the splitting of datasets for all methods, and set the same annotation budget for \Algnameabbr{}{} and FAL. 
The mitigation performance is indicated by the fairness-accuracy curves~\cite{chuang2021fair}, where the hyperparameter $\lambda$ of \Algnameabbr{} varies in the range of $(0, 2]$, and the hyperparameter setting of baseline methods can be referred to Appendix~\ref{sec:hyper_param_appendix}.
We give the fairness-accuracy curves of each method on the five benchmark datasets in Figures~\ref{fig:exp_sota_results}~(a)-(f), respectively, where \includegraphics[scale=0.012]{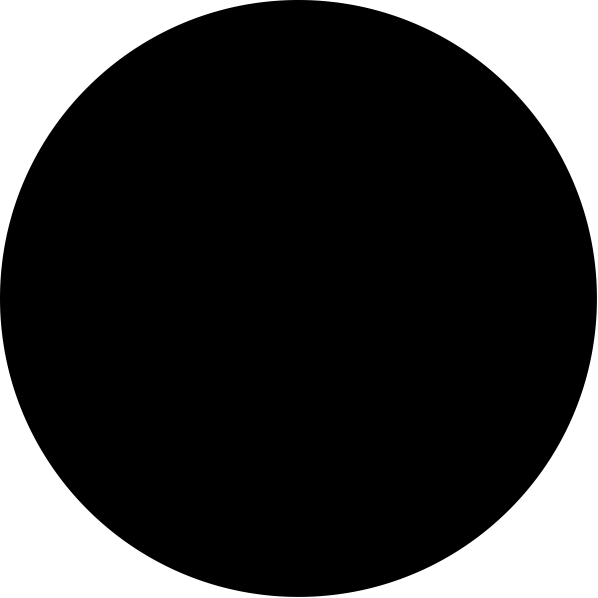}, \includegraphics[scale=0.012]{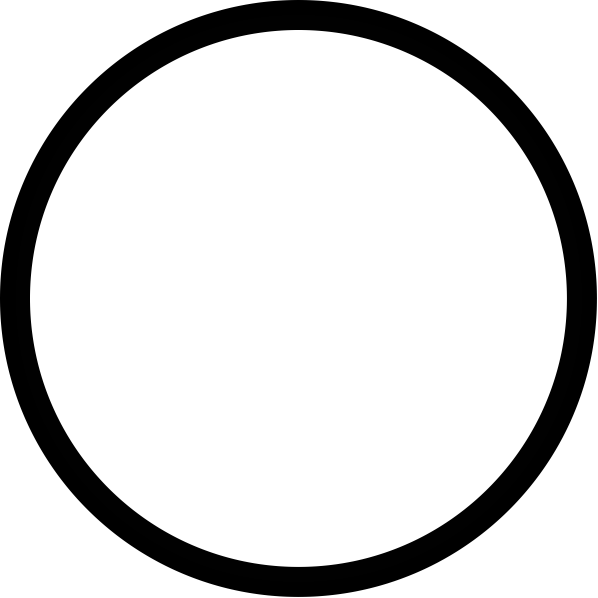} and \includegraphics[scale=0.012]{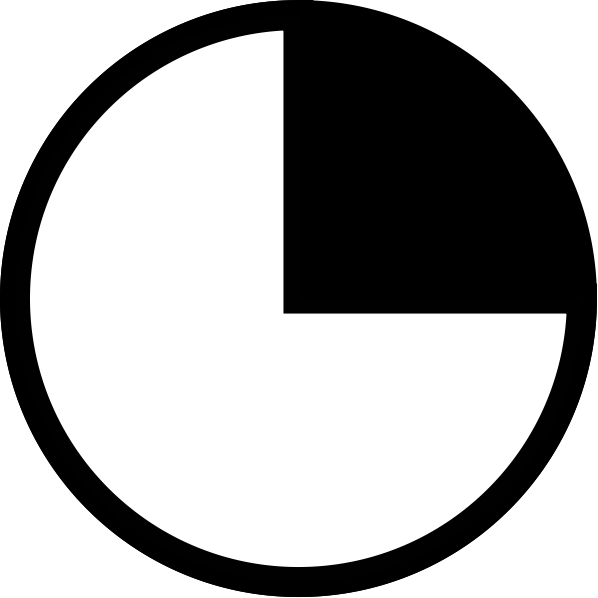} indicate the bias mitigation relies on entire-, zero- or partial- annotation of the training dataset, respectively.
Finally, we follow existing work~\cite{bechavod2017penalizing} to evaluate mitigation performance using the fairness metric $\text{EOP}$ on the MEPS, German credit and Loan default datasets, and using the fairness metric $\Delta \text{EO}$ on the remaining datasets~\cite{du2021fairness}.
We have the following observations:
\begin{itemize}[leftmargin=10pt, topsep=2mm]



\item[$\bullet$]  \Algnameabbr{} outperforms FAL on the five datasets under the same annotation budget in terms of the mitigation performance at the same level of accuracy. 
This demonstrates the superiority of \Algnameabbr{} applied to the scenarios with limited sensitive information.



\item[$\bullet$]  \Algnameabbr{} needs very few~(less than 3\% of the dataset) sensitive annotations, and shows comparable mitigation performance to Group DRO~(Group DRO requires a fully annotated dataset). 
This indicates the capacity of \Algnameabbr{} for bias mitigation under a limitation of sensitive annotations.


\item[$\bullet$]  \Algnameabbr{} outperforms LfF which relies on the proxy annotation of sensitive attributes.
It indicates that the limited human-annotated sensitive information in our framework is more beneficial than proxy annotations on the entire dataset to bias mitigation.


\end{itemize}

\begin{figure*}[t]
\setlength{\abovecaptionskip}{1mm}
\setlength{\belowcaptionskip}{-2mm}
\centering
\begin{minipage}{0.32\linewidth}
\centering
\subfigure[MEPS.]{
\centering
\includegraphics[width=0.93\textwidth]{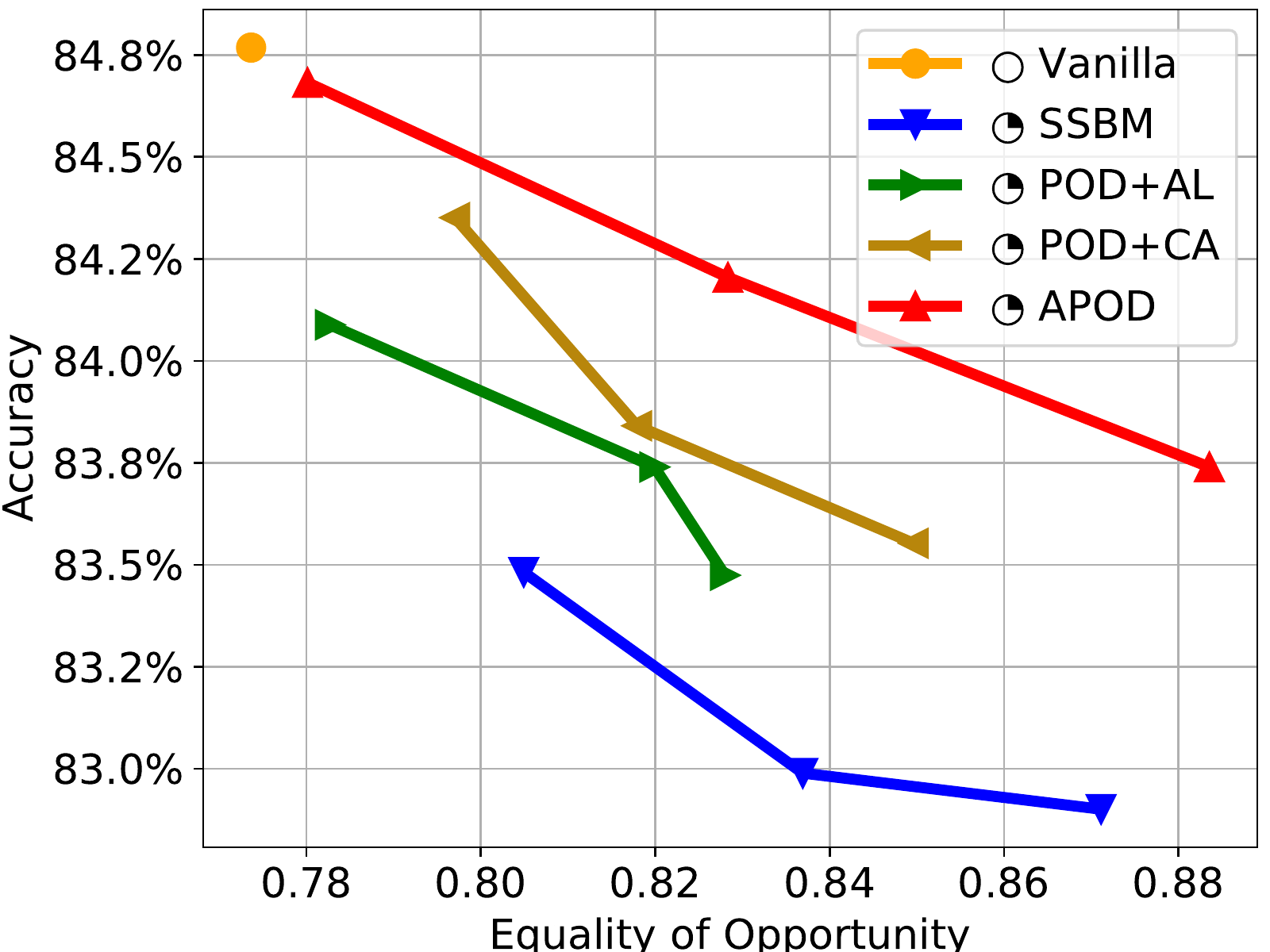}}
\end{minipage}
\begin{minipage}{0.32\linewidth}
\centering
\subfigure[German credit.]{
\centering
\includegraphics[width=0.93\textwidth]{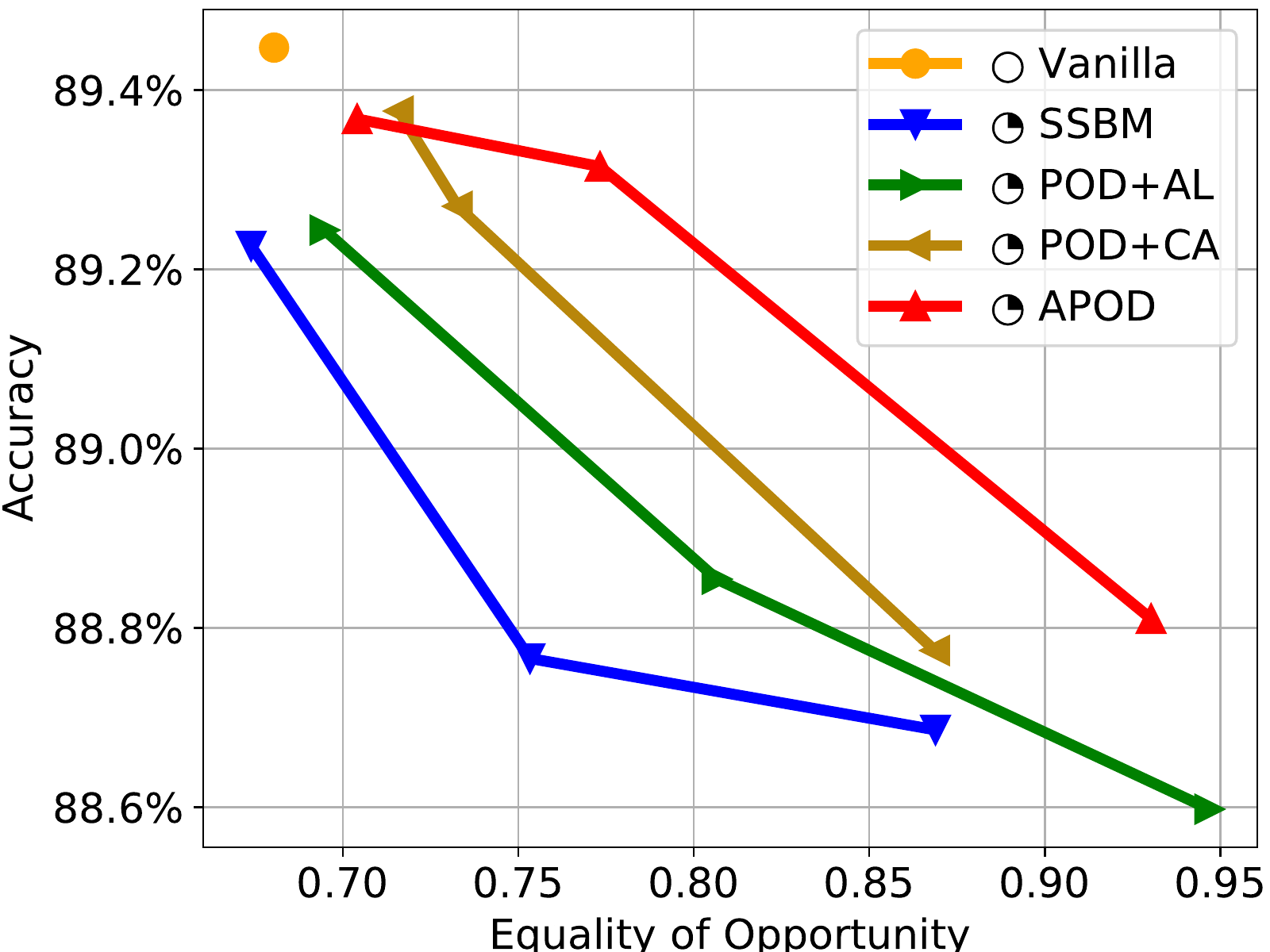}}
\end{minipage}
\begin{minipage}{0.32\linewidth}
\centering
\subfigure[Loan default.]{
\centering
\includegraphics[width=0.99\textwidth]{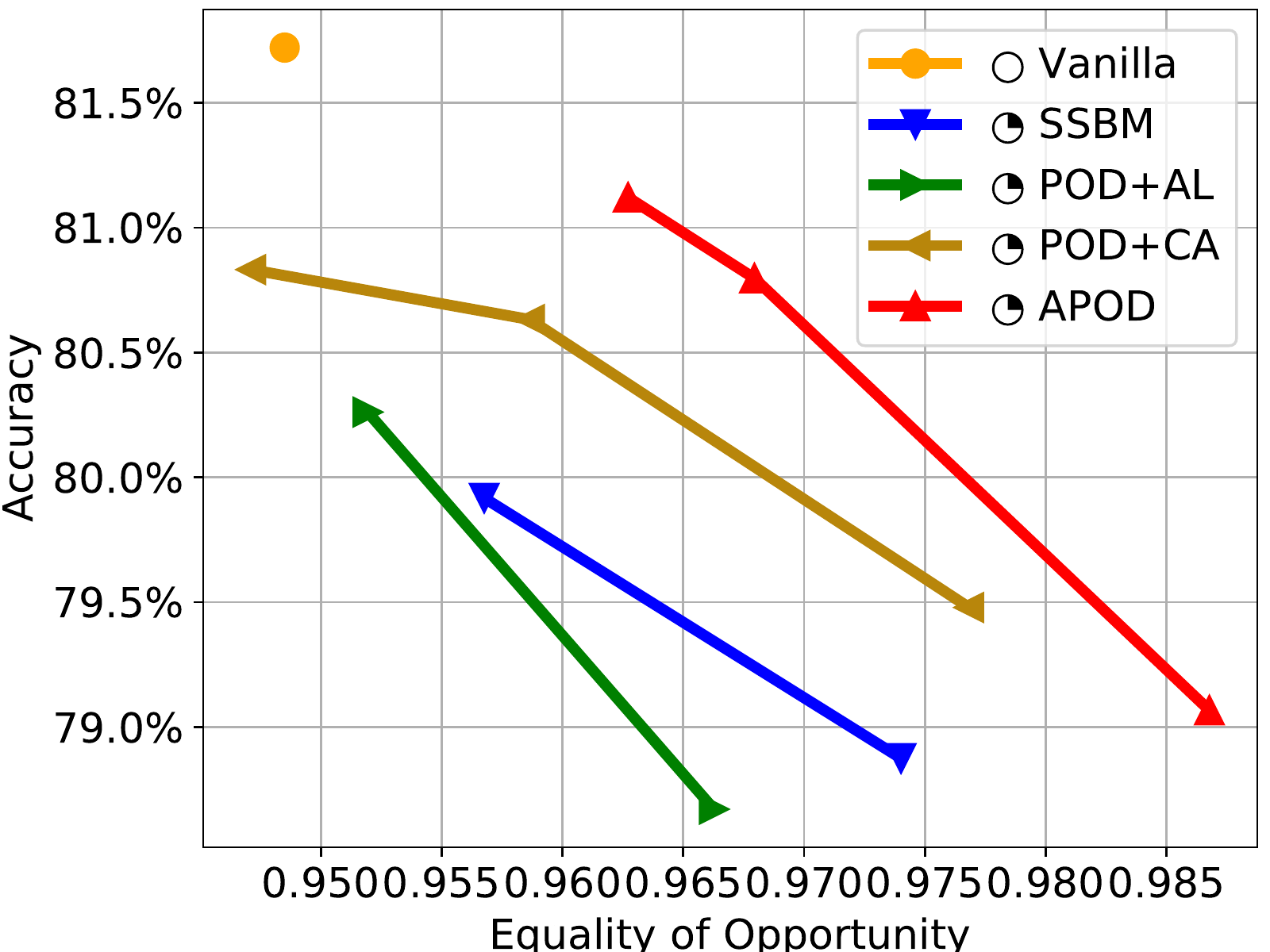}}
\end{minipage}
\begin{minipage}{0.32\linewidth}
\centering
\subfigure[Adult.]{
\centering
\includegraphics[width=0.99\textwidth]{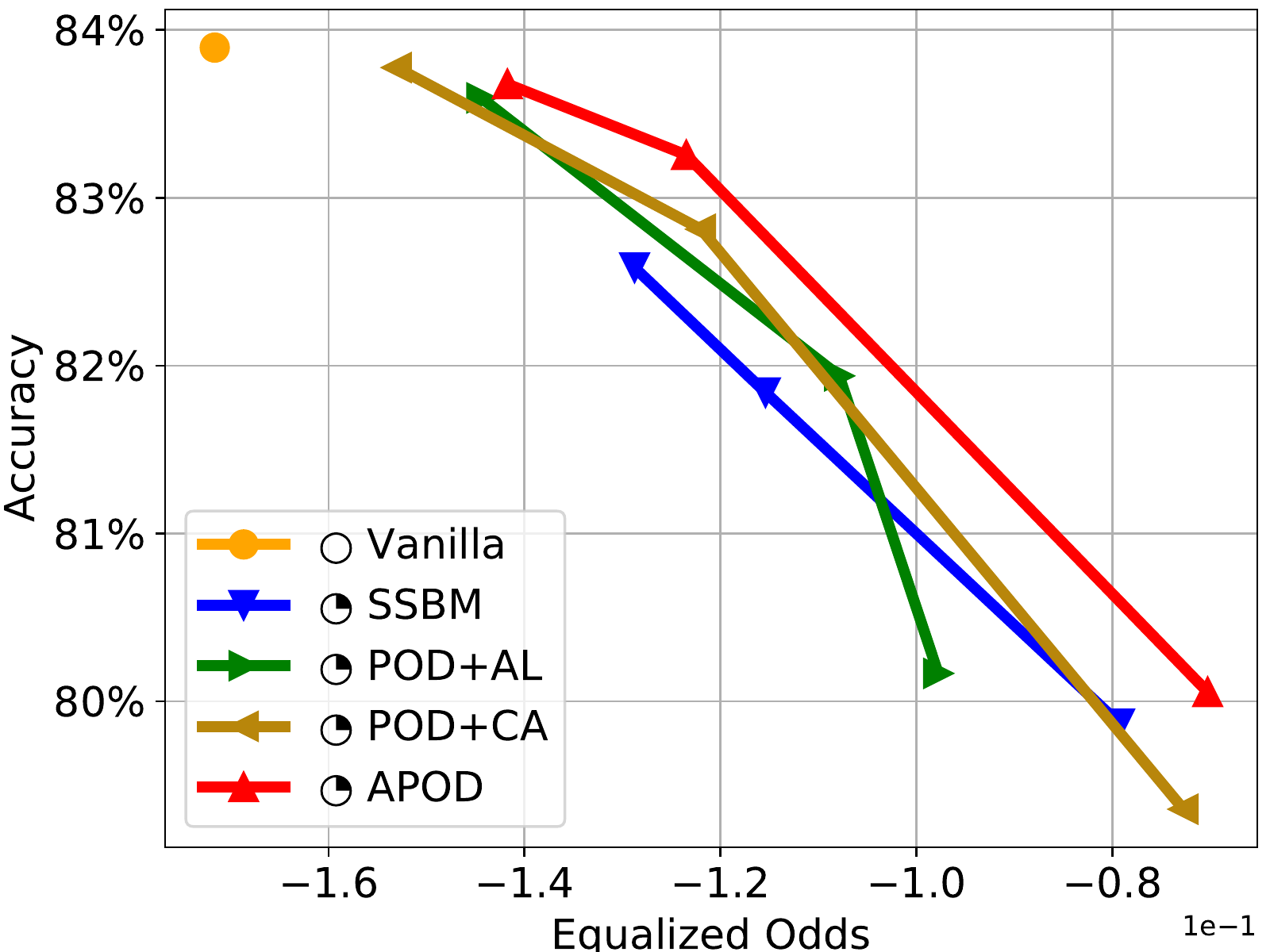}}
\end{minipage}
\begin{minipage}{0.32\linewidth}
\centering
\subfigure[CelebA-wavy hair.]{
\centering
\includegraphics[width=0.99\textwidth]{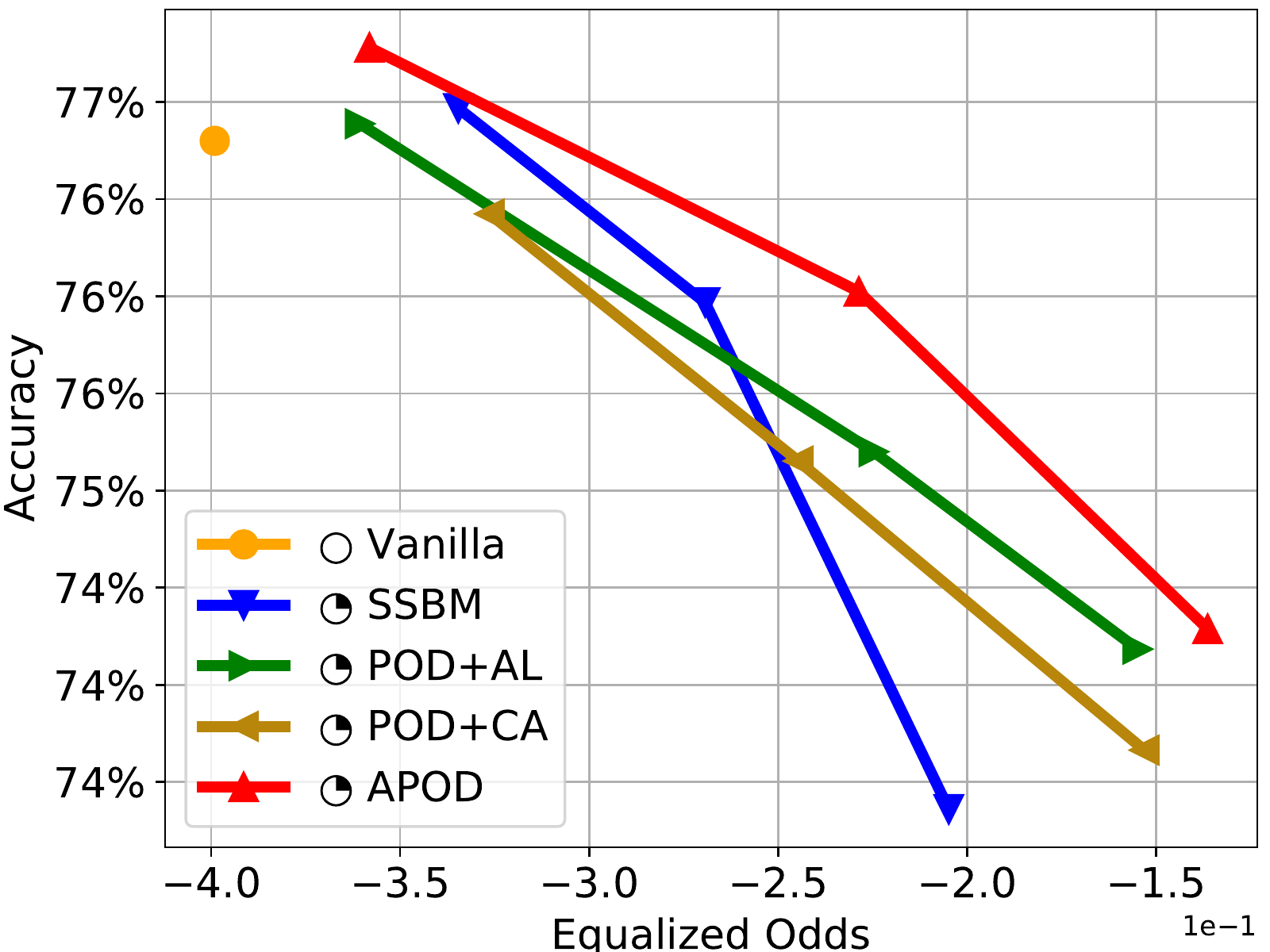}}
\end{minipage}
\begin{minipage}{0.32\linewidth}
\centering
\subfigure[CelebA-young.]{
\centering
\includegraphics[width=0.93\textwidth]{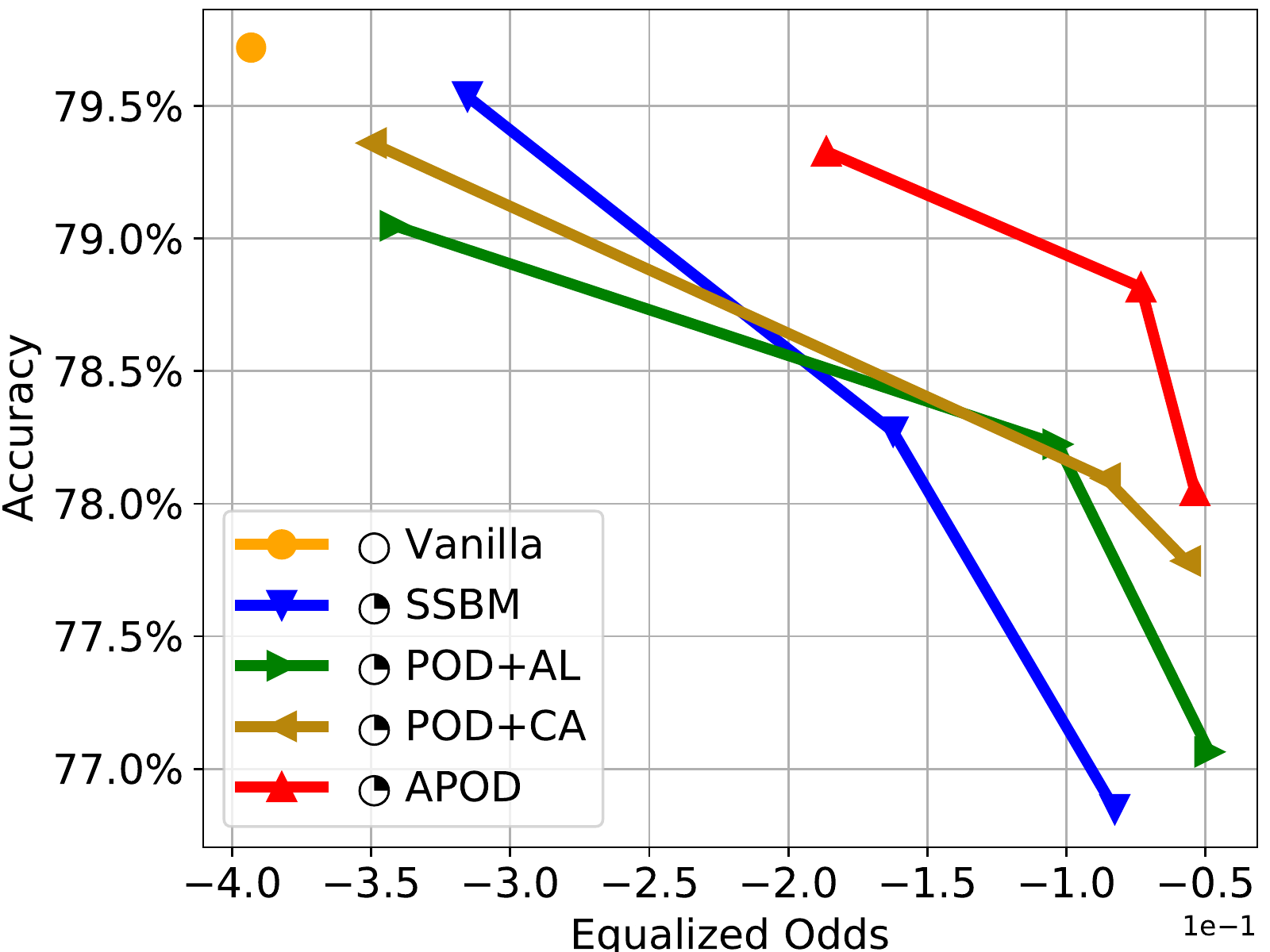}}
\end{minipage}
\caption{Accuracy-fairness curve; Algorithm: Vanilla training, SSBM, POD + AL, POD + CA and \Algnameabbr{}; Dataset: 
(a) MEPS, 
(b) German credit, 
(c) Loan default,
(d) Adult,
(e) CelebA-wavy hair,
(f) CelebA-young.
}
\label{fig:exp_al_results}
\end{figure*}

\subsection{Annotation Effectiveness Analysis~(RQ2)}

In this section, \Algnameabbr{} is compared with a semi-supervised method and two state-of-the-art active learning methods to demonstrate that AIS contributes to more informative sensitive annotations for the bias mitigation.
The key components of the baseline methods are given as follows.
\textbf{Vanilla}: The classifier is trained to minimize the cross-entropy loss without bias mitigation.
\noindent
\textbf{SSBM}: The semi-supervised bias mitigation initially samples a data subset for annotations via random selection, then adopts POD to debias the classifier on the partially annoatated dataset.
\noindent
\textbf{POD+Active learning with uncertainty sampling}~(POD+AL): 
The AIS in \Algnameabbr{} is replaced by active learning with uncertainty sampling, where an instance is selected to maximize the Shannon entropy of model prediction.
\noindent
\textbf{POD+Active learning with Core-set Approach}~(POD+CA):  
AIS is replaced by active learning with core-set approach, where an instance is selected to maximize the coverage of the entire unannotated dataset.
More details are given in the Appendix~\ref{sec:baseline_appendix}.


To unify the experiment condition, all methods have the same annotation budget and have $\lambda$ in the range of $(0,2]$.
The fairness-accuracy curves on the five datasets are given in Figures~\ref{fig:exp_al_results}~(a)-(f), respectively.
According to the mitigating results, we have the following observations:
\begin{itemize}[leftmargin=10pt, topsep=4mm]


\item[$\bullet$]  Compared to the semi-supervised method and the active learning-based methods, \Algnameabbr{} achieves better mitigation performance at the same level of accuracy, indicating the proposed AIS selects more informative annotations than those methods for bias mitigation.

\item[$\bullet$]  Different from POD+AL and POD+CA which sample the annotated instances from the whole dataset in each iteration, \Algnameabbr{} interactively selects more representative instances from different subgroups in different iterations, i.e. $\mathscr{U}_a^y$ for $a \in \mathcal{A}$ and $y \in \mathcal{Y}$, which contributes to more effective bias mitigation.



\item[$\bullet$]  SSBM shows almost the worst mitigation performance among all of the methods, because the initially randomly selected subset preserves the skewness of the original dataset, leading to non-optimal bias mitigation, which is consistent with our discussion in Section~\ref{sec:intro}.






\end{itemize}

\begin{figure*}[t]
\setlength{\abovecaptionskip}{0mm}
\setlength{\belowcaptionskip}{-3mm}
\centering
\begin{minipage}{0.32\linewidth}
\centering
\subfigure[Adult.]{
\centering
\includegraphics[width=1.0\textwidth]{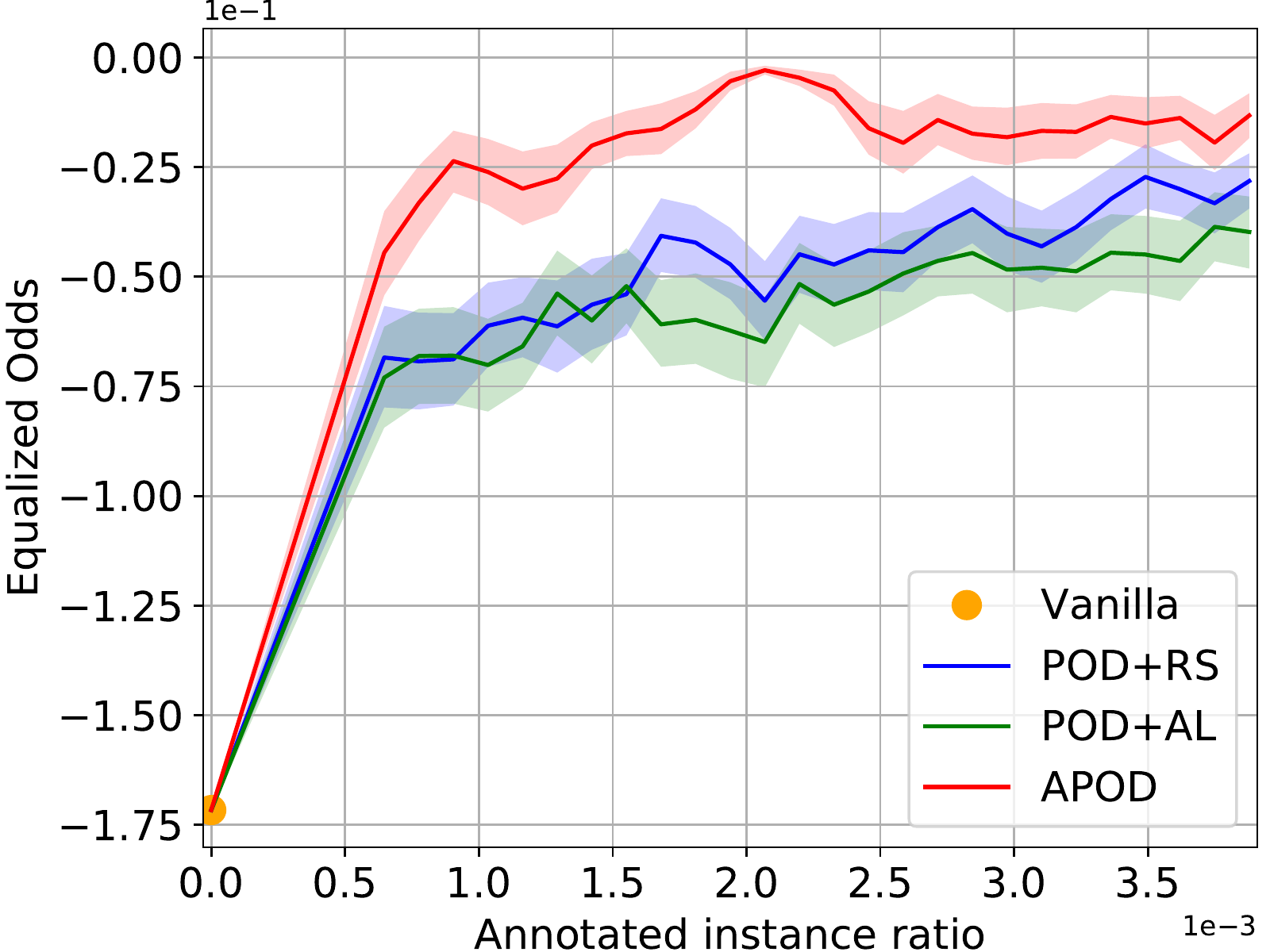}}
\end{minipage}
\begin{minipage}{0.32\linewidth}
\centering
\subfigure[Loan default.]{
\centering
\includegraphics[width=1.0\textwidth]{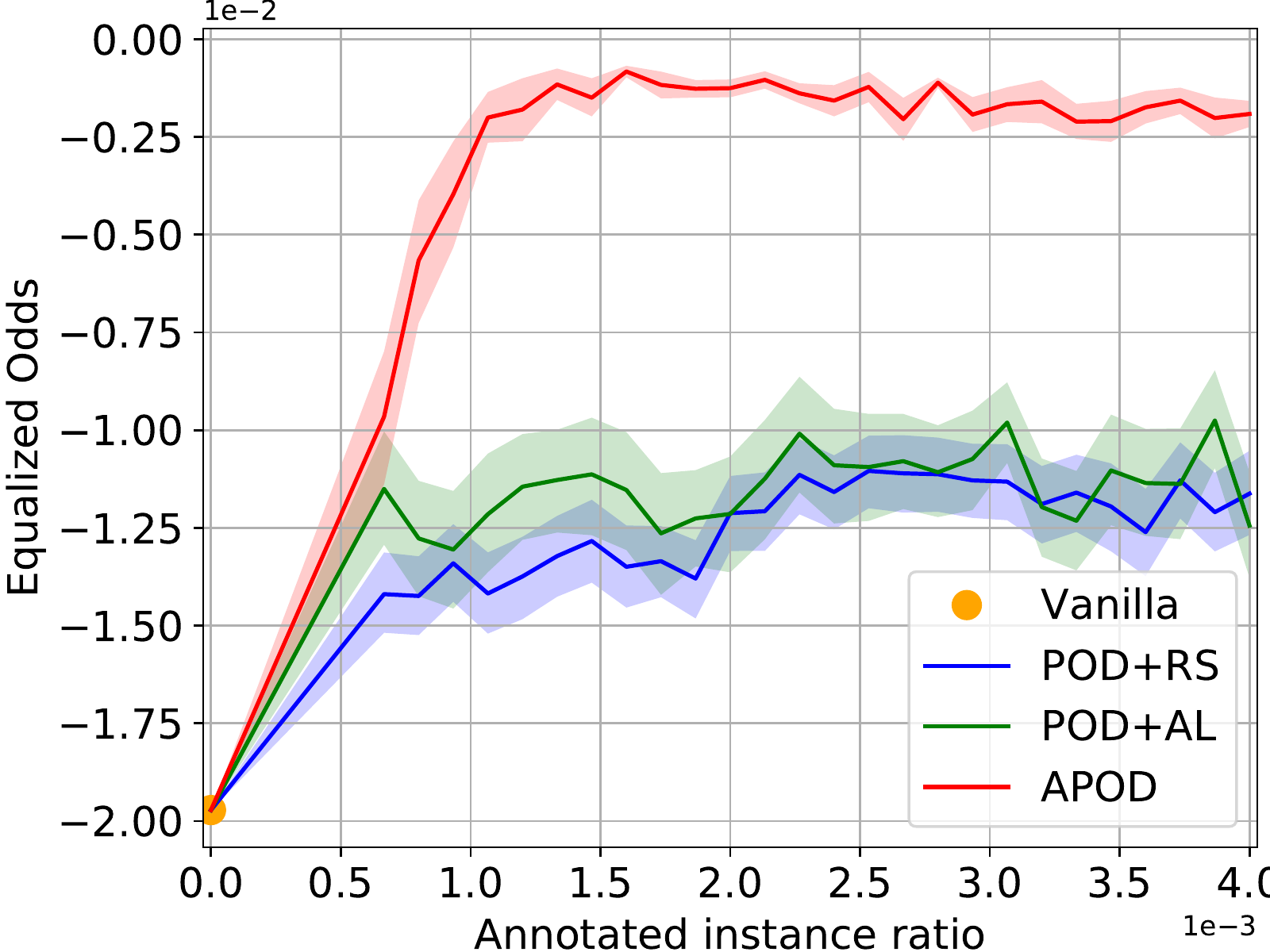}}
\end{minipage}
\begin{minipage}{0.32\linewidth}
\centering
\subfigure[MEPS. Ablation result.]{
\centering
\includegraphics[width=1.0\textwidth]{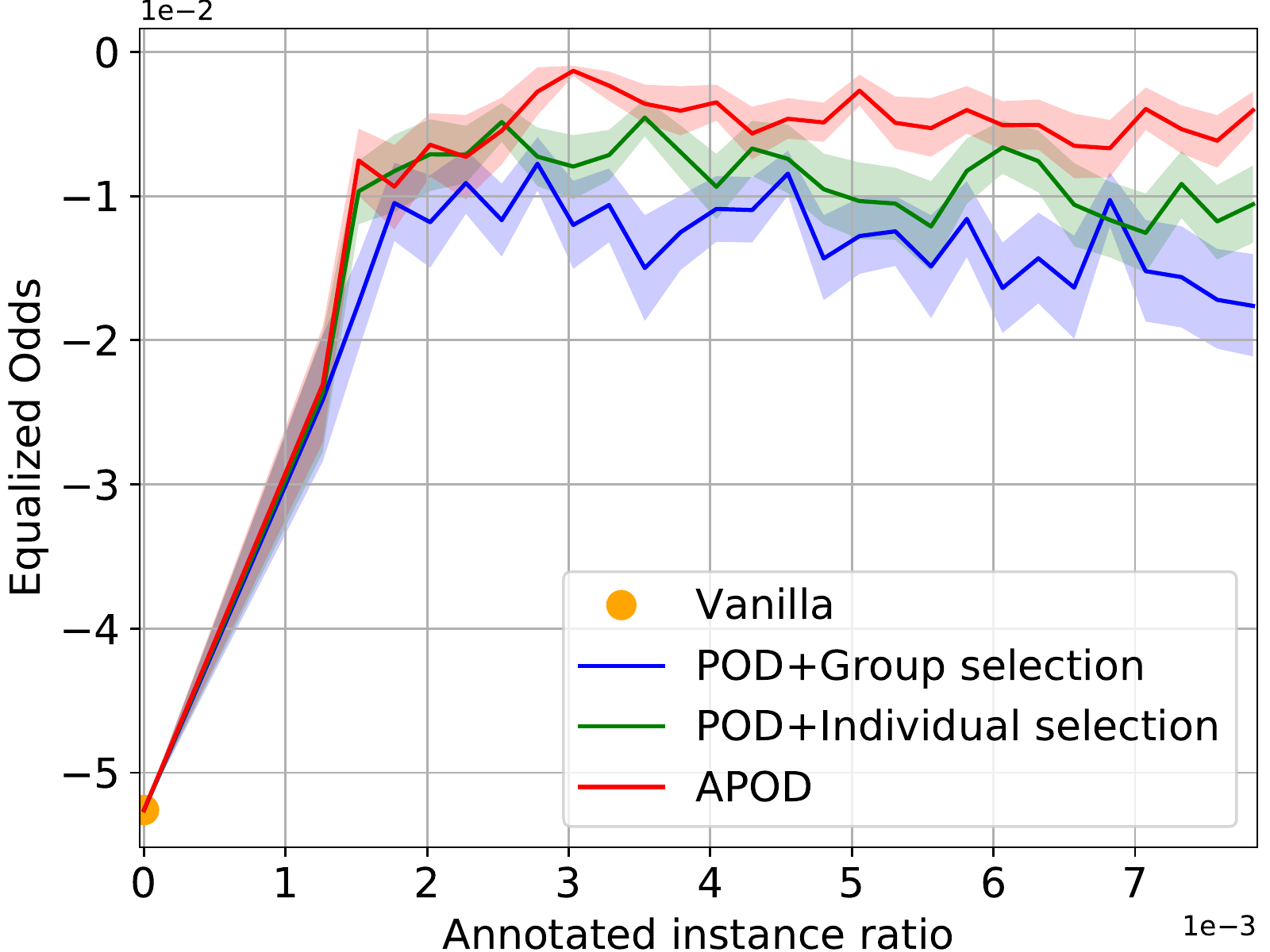}}
\end{minipage}
\caption{Effect of the annotation ratio to \Algnameabbr{}, POD+RS and POD+AL on (a) Adult and (b) Loan default dataset. 
(c) Mitigation performance of \Algnameabbr{}, POD+Group selection and POD+Individual selection.
}
\label{fig:exp_results2}
\end{figure*}




\vspace{-5mm}
\subsection{Annotation Ratio Analysis (RQ3)}

We now evaluate the effect of the the annotation ratio (that is, the ratio of the annotated instances to the training instances) on bias mitigation.
Specifically, we tune $\lambda$ in the range of $(0, 2]$, and find that $\lambda = 0.5$ and $0.1$ can provide a good accuracy-fairness trade-off on the Adult and Loan default datasets, respectively.
In addition to the existing baseline methods, we also consider replacing AIS in \Algnameabbr{} into random selection~(\textbf{POD+RS}) for comparison.
Since one instance is selected for annotation in each iteration of \Algnameabbr{}, the Equalized Odds of the snapshot model in each iteration is estimated and plotted versus the annotation ratio on the Adult and Loan default datasets in Figures~\ref{fig:exp_results2}~(a) and (b), respectively.
We also give the error bar to show the standard deviation of each method.
Overall, we have the following observations:
\begin{itemize}[leftmargin=10pt, topsep=5pt]

    \item[$\bullet$]  All methods achieve better mitigation performance as the annotation ratio increases due to the distribution of the annotated set becoming consistent with the entire dataset. 

    \item[$\bullet$]  \Algnameabbr{} shows better mitigation performance than POD+AL and POD+RS at the same level of annotation ratios. 
    This indicates the selection of annotated instances by AIS significantly leads to a reduction of bias. 
    In contrast, the bias mitigation of POD+RS merely derives from the increasing annotations.
    
    

    \item[$\bullet$]  \Algnameabbr{} shows significantly higher improvement in bias mitigation than the baseline methods even when the annotation ratio is small, and enables the mitigation to converge to a higher level at smaller annotation ratios~(i.e., earlier) than baseline methods.

\end{itemize}



\subsection{Ablation Study (RQ4)}

To demonstrate the effectiveness of group selection and individual selection, \Algnameabbr{} is compared with two compositional methods: POD+Group selection and POD+Individual selection.
The three methods are tested with the same hyperparameter setting on the MEPS dataset.
The value of the fairness metric is given in Figure~\ref{fig:exp_results2}~(c).
It is observed that both POD+Group selection and POD+Individual selection show considerable degradation in mitigation performance compared to \Algnameabbr{}. It empirically validates Remark~\ref{sec:rk1} that both group selection and individual selection in AIS contribute to tightening the upper bound of the relaxed fairness metrics, thus contributing to bias mitigation.



    
    
    


\subsection{Visualization of Annotated Instances}

\label{sec:exp_visualization}
We visualize the tSNE embeddings of the annotated instances to trace the active instance selection of \Algnameabbr{}.
The tSNE visualization is given in Figures~\ref{fig:annotated_instances}~(a)-(d) in Appendix~\ref{sec:visualization_appendix}.
Specifically, Figures~\ref{fig:annotated_instances}~(a)-(d) illustrate the the tSNE embeddings of the annotated instances selected by \Algnameabbr{} and random selection on the MEPS and Adult datasets, respectively.
We use different colors to indicate different groups, where positive instances (Y=1) are less than negative ones (Y=0), and the unprivileged group (A=0) is smaller than the privileged group (A=1).
Overall, we have the following observations:

\begin{itemize}[leftmargin=10pt, topsep=5pt]

    \item[$\bullet$]  The annotated instances of Random selection in Figures~\ref{fig:annotated_instances}~(a) and (c) are consistent with Figure~\ref{fig:unfair_vs_rs_vs_APD}~(b), which follows the skewed distribution of original dataset, and leads to non-optimal mitigation of bias. 
    
    \item[$\bullet$]  The annotated instances of \Algnameabbr{} in Figures~\ref{fig:annotated_instances}~(b) and~(d) are consistent with the optimal annotating in Figure~\ref{fig:unfair_vs_rs_vs_APD}~(c), where the annotated subset shows less skewness compared to the original distribution. 
    
    \item[$\bullet$]  \Algnameabbr{} selects more annotated instances from the unprivileged group $\{(x_i, \! y_i, \! a_i) \\ | y_i \!=\! 1, a_i \!=\! 0\}$ than random selection.
    In such a manner, \Algnameabbr{} improves the contribution of unprivileged group to the average loss, which contributes to the bias mitigation.
    
\end{itemize}

\section{Conclusion}

In this paper, we propose \Algnameabbr{}, an iterative framework for active bias mitigation under the limitation of sensitive annotations.
In each iteration, \Algnameabbr{} guides the active instance selection to discover the optimal instance for annotation, and maximally promotes bias mitigation based on the partially annotated dataset through penalization of discrimination. 
Theoretical analysis indicates that \Algnameabbr{} contributes to effective bias mitigation via bounding the relaxed fairness metrics.
Experiment results further demonstrate the effectiveness of \Algnameabbr{} on five benchmark datasets, where it outperforms baseline methods under the same annotation budget and has a desirable outcome of bias mitigation even when most of the sensitive annotations are unavailable.
This also indicates its benefit to real-world applications, especially when the disclosed or available sensitive information is very limited.



\bibliographystyle{splncs04}
\bibliography{apod_reference}

\newpage
\appendix
\section*{Appendix}

\setcounter{theorem}{0}

\section{Visualization of Annotated Instances}
\label{sec:visualization_appendix}

The tSNE visualization is given in Figures~\ref{fig:annotated_instances}~(a)-(d).
We have several observations on this results which can be referred to Section~\ref{sec:exp_visualization}.

\begin{figure*}
\setlength{\abovecaptionskip}{2mm}
\setlength{\belowcaptionskip}{0mm}
\centering
\begin{minipage}{0.24\linewidth}
\centering
\subfigure[\scriptsize Random selection.]{
\centering
\includegraphics[width=1.0\textwidth]{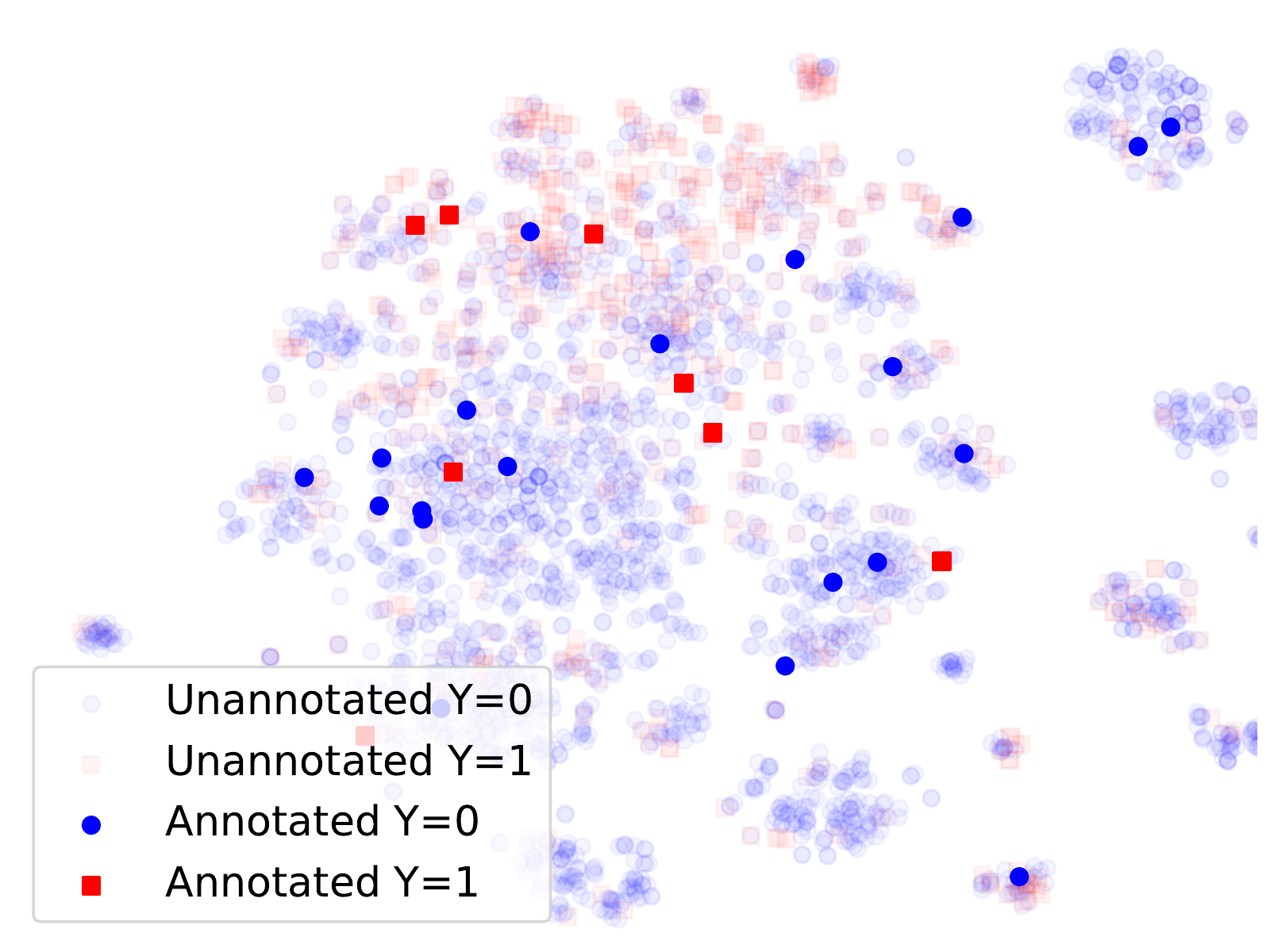}}
\end{minipage}
\begin{minipage}{0.24\linewidth}
\centering
\subfigure[\scriptsize \Algnameabbr{}.]{
\centering
\includegraphics[width=1.0\textwidth]{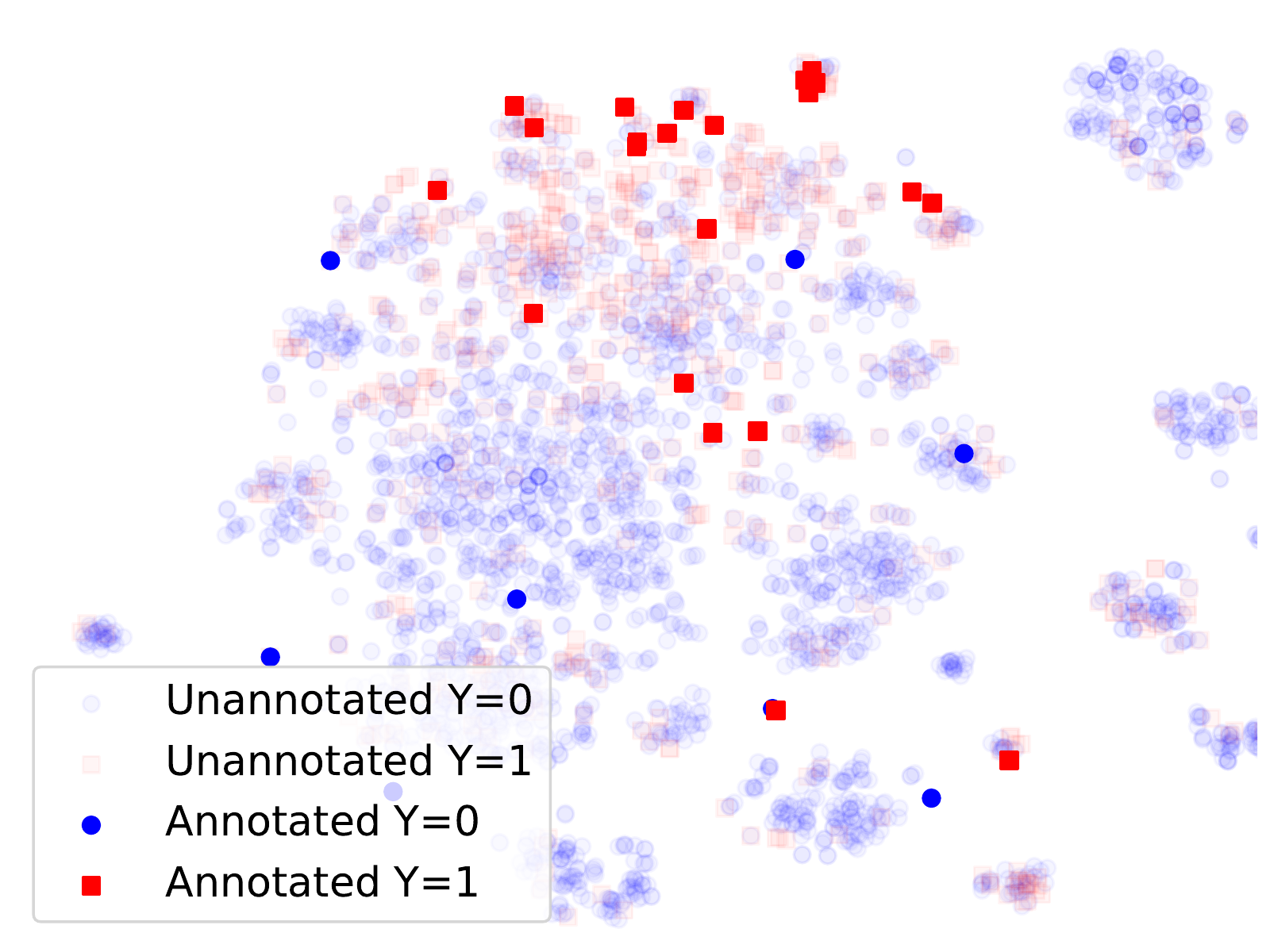}}
\end{minipage}
\begin{minipage}{0.24\linewidth}
\centering
\subfigure[\scriptsize Random selection.]{
\centering
\includegraphics[width=1.0\textwidth]{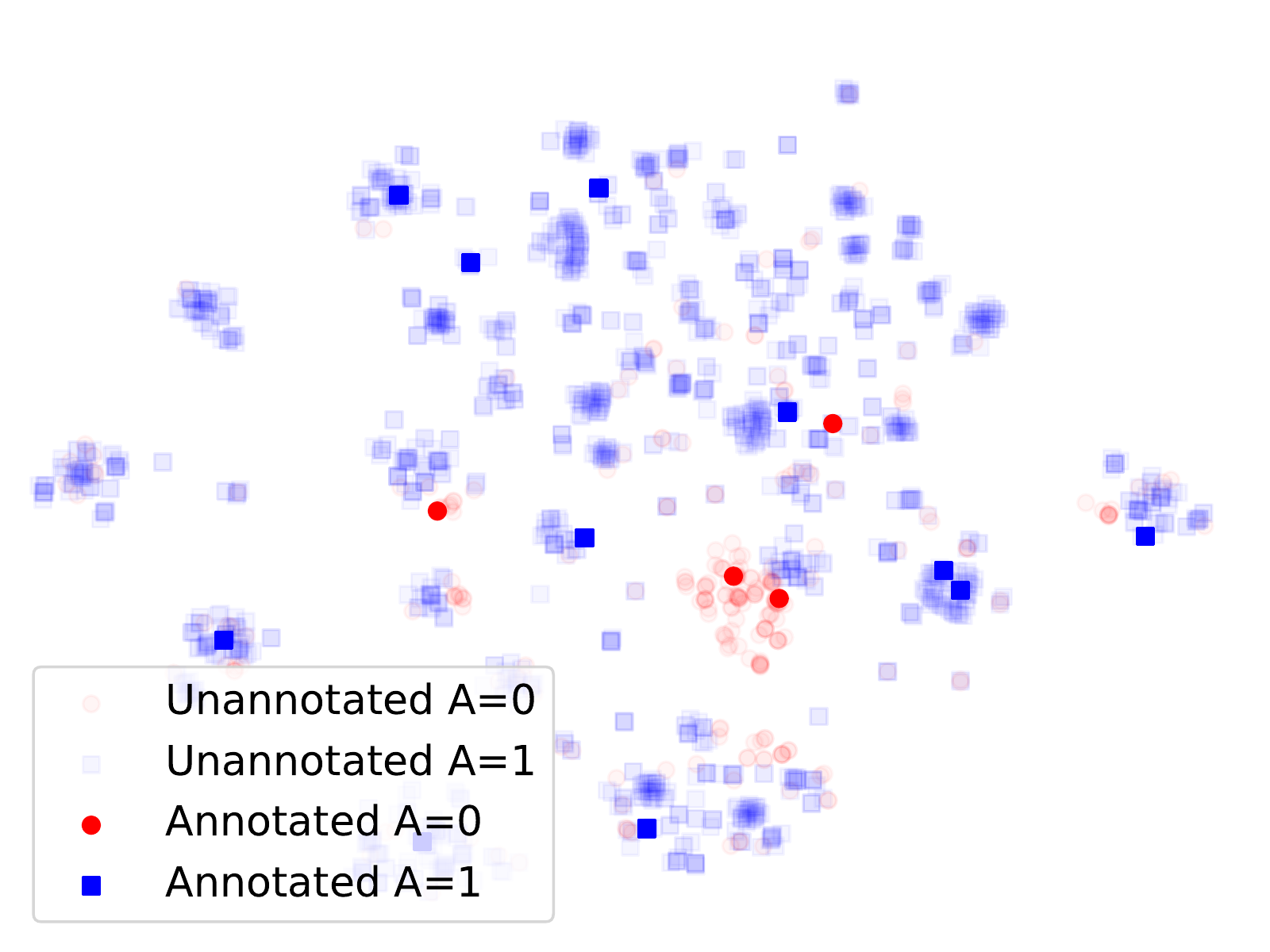}}
\end{minipage}
\begin{minipage}{0.24\linewidth}
\centering
\subfigure[\scriptsize \Algnameabbr{}.]{
\centering
\includegraphics[width=1.0\textwidth]{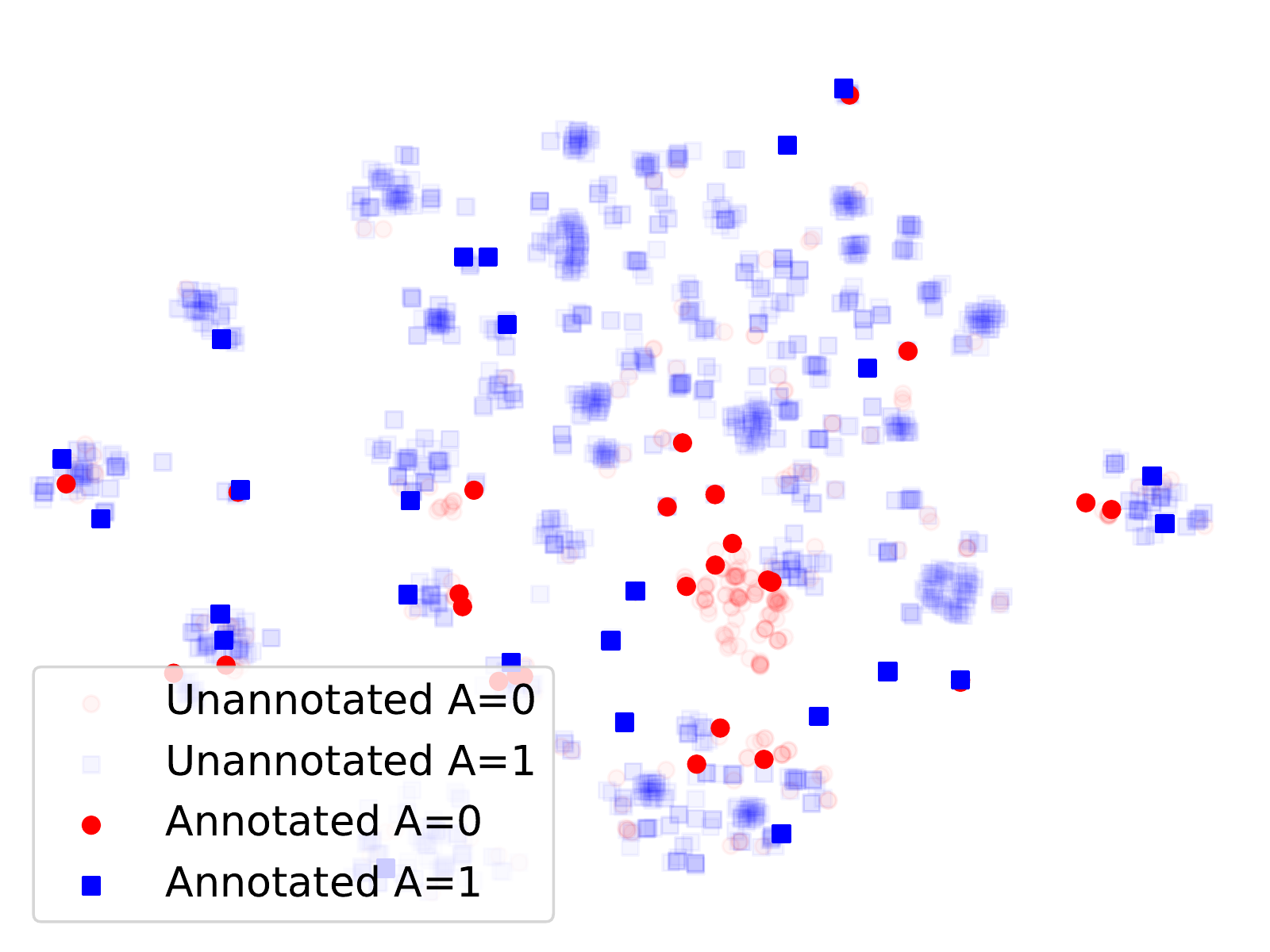}}
\end{minipage}
\caption{Comparison of \Algnameabbr{} and Random selection in terms of the annotated instances from different groups.
(a) Annotated instances by Random selection.
(b) Annotated instances by \Algnameabbr{}.
(c) Annotated positive instances (Y=1) by Random selection.
(d) Annotated positive instances (Y=1) by \Algnameabbr{}.
}
\label{fig:annotated_instances}
\end{figure*}

\section{Proof of Theorem~\ref{pp:upper_bound}}
\label{sec:proof_theorem_appendix}

In this section, we first propose Corollary~\ref{prop:prob_ineq}, then adopt the corollary to prove Theorem~\ref{pp:upper_bound}.

\begin{corollary}
\label{prop:prob_ineq}
For $p(y), f(y), g(y) > 0$, we have
\begin{equation}
\begin{aligned}
\int_{\mathcal{Y}} p(y) f(y) \mathrm{d} y &\leq \int_{\mathcal{Y}} q(y) f(y) \mathrm{d} y + \int_{\mathcal{Y}} |p(y) - q(y)| f(y) \mathrm{d} y
\\
\int_{\mathcal{Y}} p(y) f(y) \mathrm{d} y &\leq \int_{\mathcal{Y}} p(y) g(y) \mathrm{d} y + \int_{\mathcal{Y}} p(y) |f(y) - g(y)|  \mathrm{d} y
\nonumber
\end{aligned}
\end{equation}
\end{corollary}

\begin{proof}
\begin{equation}
\begin{aligned}
\int_{\mathcal{Y}} p(y) f(y) \mathrm{d} y &= \int_{\mathcal{Y}} q(y) f(y) \mathrm{d} y + \int_{\mathcal{Y}} [p(y) - q(y)] f(y) \mathrm{d} y
\\
&\leq \int_{\mathcal{Y}} q(y) f(y) \mathrm{d} y + \int_{\mathcal{Y}} |p(y) - q(y)| f(y) \mathrm{d} y
\\
\int_{\mathcal{Y}} p(y) f(y) \mathrm{d} y &=
\int_{\mathcal{Y}} p(y) g(y) \mathrm{d} y + \int_{\mathcal{Y}} p(y) [f(y) - g(y)] \mathrm{d} y
\\
&\leq \int_{\mathcal{Y}} p(y) g(y) \mathrm{d} y + \int_{\mathcal{Y}} p(y) |f(y) - g(y)| \mathrm{d} y
\nonumber
\end{aligned}
\end{equation}
\end{proof}

\noindent
After proving Corollary~\ref{prop:prob_ineq}, we return to prove the theorem.

\begin{theorem}
Assume the loss value on the training set satisfies $\frac{1}{|\mathcal{S}|} \sum_{(\boldsymbol{x}_i, y_i, a_i) \in \mathcal{S}} \\ l(\boldsymbol{h}_i, y_i;  \theta_h) \leq \epsilon$\footnote{$\epsilon$ is small if the classifier head $f_h$ has been well-trained on the annotated dataset $\mathcal{S}$.}, 
and $l(\boldsymbol{h}, y; \theta_h)$ and $f_h$ satisfy $K_l$- and $K_h$-Lipschitz continuity\footnote{$l(\boldsymbol{h}, y;  \theta_h)$ and $f_h$ satisfy $|l(\boldsymbol{h}_i, y; \theta_h) - l(\boldsymbol{h}_j, y; \theta_h)| \leq K_l ||\boldsymbol{h}_{i} - \boldsymbol{h}_{j}||_2$ and $|p(y | \boldsymbol{x}_i) - p(y | \boldsymbol{x}_j)| \leq K_h ||\boldsymbol{h}_{i} - \boldsymbol{h}_{j}||_2$, respectively, where the likelihood function $p(y \mid \boldsymbol{x}_i) = \text{softmax}(f_h(\boldsymbol{h}_i | \theta_h))$.}, respectively.
The generalization loss difference between unprivileged group and privileged group has the following upper bound with probability $1-\gamma$, 
\begin{align}
&\bigg| \! \int_{\mathcal{X}_0} \!\! \int_{\mathcal{Y}} \! p(\boldsymbol{x},  y) l(\boldsymbol{h},  y;  \theta_h) \mathrm{d}\boldsymbol{x} \mathrm{d}y \!-\!\! \int_{\mathcal{X}_1} \!\! \int_{\mathcal{Y}} \! p(\boldsymbol{x},  y) l(\boldsymbol{h},  y;  \theta_h) \mathrm{d}\boldsymbol{x} \mathrm{d}y \bigg|
\nonumber
\\
&\quad\quad\quad\quad\quad\quad\quad\quad\quad \leq \epsilon + \min \Big\{ \sqrt{-L^2 \log \gamma (2 N_{\tilde{a}})^{-1}},  (K_l + K_h L) \delta_{\tilde{a}} \Big\}, 
\end{align}
where $\tilde{a} \!=\! \mathop{\arg\max}_{a \in \mathcal{A}}  \int_{\mathcal{X}_a}   \int_{\mathcal{Y}}  p(\boldsymbol{x}, y) l(\boldsymbol{h}, y; \theta_h) \mathrm{d}\boldsymbol{x} \mathrm{d}y$; $\mathcal{X}_a \!=\! \{\boldsymbol{x}_i \in \mathscr{D} | a_i \!=\! a \}$;
$\delta_{\tilde{a}} \!=\! \max_{\boldsymbol{x}_i \in \mathcal{X}_{\tilde{a}}} \!\! \min_{(\boldsymbol{x}_j, y_j, a_j) \in \mathcal{S}} \!\! || \boldsymbol{h}_i \!-\! \boldsymbol{h}_j ||_2$;  $N_{\tilde{a}} \!=\! |\{ (\boldsymbol{x}_i, y_i, a_i) | a_i \!=\! \tilde{a}, (\boldsymbol{x}_i, y_i, a_i) \!\in\! \mathcal{S} \}|$;
$L \!=\! \max_{(\boldsymbol{x}_i, y_i) \in \mathscr{U}} l(\boldsymbol{h}_i, y_i; \theta_h)$; 
and $\boldsymbol{h}_i = f_b(\boldsymbol{x}_i | \theta_b)$.

\end{theorem}


\begin{proof}
According to the upper bound of generalization error, the generalization error for group $\boldsymbol{x}\in\mathcal{X}_a$ for $\forall a \in \mathcal{A}$ is bounded with probability $1-\gamma$,
\begin{equation}
\label{eq:Hoeffding_ineq}
g_a \!=\! \int_{\mathcal{X}_a} \! \int_{\mathcal{Y}} p(\boldsymbol{x}, y) l(\boldsymbol{h}, y; \theta_h) \mathrm{d}\boldsymbol{x} \mathrm{d}y \!\leq\! \epsilon + \sqrt{-L^2 \log \gamma (2N_a)^{-1}},
\nonumber
\end{equation}
where $L \!=\! \max_{(\boldsymbol{x}_i, y_i) \in \mathscr{U}} l(\boldsymbol{h}_i, y_i; \theta_h)$.
Moreover, we consider the upper bound of absolute gap
$|g_0 - g_1| \leq \max_{a\in\mathcal{A}} g_a$.
The generalization error difference between the two groups is bounded with probability $1-\gamma$ as follow,
\begin{equation}
\begin{aligned}
\label{eq:gap_bound1}
&\bigg| \int_{\mathcal{X}_0} \! \int_{\mathcal{Y}} \! p(\boldsymbol{x}, \! y) l(\boldsymbol{h}, y; \theta_h) \mathrm{d}\boldsymbol{x} \mathrm{d}y \!-\!\! \int_{\mathcal{X}_1} \!\! \int_{\mathcal{Y}} \! p(\boldsymbol{x}, \! y) l(\boldsymbol{h}, y; \theta_h) \mathrm{d}\boldsymbol{x}  \mathrm{d}y \bigg|
\\
&\quad\quad\quad\quad\quad\quad\quad\quad\quad \leq \epsilon + \sqrt{-L^2 \log \gamma (2N_{\tilde{a}})^{-1}},
\end{aligned}
\end{equation}
where $\tilde{a} = \arg \max_{a \in \mathcal{A}} \int_{\mathcal{X}_a} \int_{\mathcal{Y}} p(\boldsymbol{x}, y) l(\boldsymbol{h}, y; \theta_h) \mathrm{d}\boldsymbol{x} \mathrm{d}y$.

To prove the second bound of the generalization error difference, let
$\mathcal{N}(\boldsymbol{x}_i)$ denote the nearest neighbour of $\boldsymbol{x}_i \in \mathcal{X}$ which belongs to the annotated dataset, i.e. $\mathcal{N}(\boldsymbol{x}_i) = \arg \min_{(\boldsymbol{x}_j, y_j, a_j) \in \mathcal{S}} || \boldsymbol{h}_{j} - \boldsymbol{h}_{i} ||_2$; 
let $\boldsymbol{h}_i^{\mathcal{N}}$ denote the embedding of $\mathcal{N}(\boldsymbol{x}_i)$;
and let $d_{\boldsymbol{x}_i}$ denote the distance between $\boldsymbol{x}_i$ and $\mathcal{N}(\boldsymbol{x}_i)$ in the embedding space, i.e. $d_{\boldsymbol{x}_i} = \min_{(\boldsymbol{x}_j, y_j, a_j) \in \mathcal{S}} || \boldsymbol{h}_i - \boldsymbol{h}_j ||_2$. 
Accoding to Corollary~\ref{prop:prob_ineq}, with $p(y) = p(y | \boldsymbol{x}_i)$, $q(y) = p(y | \mathcal{N}(\boldsymbol{x}_i))$ and $f(y) = l(\boldsymbol{h}_i, y; \theta_h)$, the generalization error can be bounded by 
\begin{equation}
\begin{aligned}
\label{eq:error_upper_bound}
&\int_{\mathcal{Y}} p(y | \boldsymbol{x}_i) l(\boldsymbol{h}_i, y; \theta_h) \mathrm{d}y \leq 
\\
&\int_{\mathcal{Y}} p(y | \mathcal{N}(\boldsymbol{x}_i)) l(\boldsymbol{h}_i, y; \theta_h) \mathrm{d}y 
+ \int_{\mathcal{Y}} \big| p(y | \boldsymbol{x}_i) - p(y | \mathcal{N}(\boldsymbol{x}_i)) \big| l(\boldsymbol{h}_i, y; \theta_h) \mathrm{d}y.
\end{aligned}
\end{equation}

Note that the classifier head $f_h$ satisfies $K_h$-Lipschitz continuity $|p(y | \boldsymbol{x}_i) - p(y | \mathcal{N}(\boldsymbol{x}_i))| \leq K_h | ||\boldsymbol{h}_{i} - \boldsymbol{h}_i^\mathcal{N}||_2 = K_h d_{\boldsymbol{x}_i}$ and $l(\boldsymbol{h}, y; \theta_h) \leq L$, the second term in the right-side of Equation~(\ref{eq:error_upper_bound}) is bounded by
\begin{equation}
\label{eq:upper_bound1}
\int_{\mathcal{Y}} \big| p(y | \boldsymbol{x}_i) - p(y | \mathcal{N}(\boldsymbol{x}_i)) \big| l(\boldsymbol{h}_i, y; \theta_h) \mathrm{d}y \leq K_h L d_{\boldsymbol{x}_i}.
\end{equation}
Furthermore, taking $p(y) \!=\! p(y | \mathcal{N}(\boldsymbol{x}_i))$, $f(y) = l(\boldsymbol{h}_i, y; \theta_h)$ and $g(y) = l(\boldsymbol{h}_i^{\mathcal{N}}, y; \theta_h)$ into Corollary~\ref{prop:prob_ineq}, we have the first term in the right-side of Equation~(\ref{eq:error_upper_bound}) can be bounded by
\begin{equation}
\begin{aligned}
\label{eq:upper_bound2}
&\int_{\mathcal{Y}} \!\! p(y | \mathcal{N}(\boldsymbol{x}_i)) l(\boldsymbol{h}, \! y; \! \theta_h) \mathrm{d}y \! \leq 
\\
&\! \int_{\mathcal{Y}} \!\! p(y | \mathcal{N}(\boldsymbol{x}_i)) l(\boldsymbol{h}_i^{\mathcal{N}}, y; \theta_h) \mathrm{d}y \! + \!\!\! \int_{\mathcal{Y}} \!\! p(y | \mathcal{N}(\boldsymbol{x}_i)) | l(\boldsymbol{h}, \! y; \! \theta_h) \!-\! l(\boldsymbol{h}_i^{\mathcal{N}}, y; \theta_h) | \mathrm{d}y 
\\
&\leq \epsilon \!+\! K_l d_{\boldsymbol{x}_i} \!,
\end{aligned}
\end{equation}
where we have $\int_{\mathcal{Y}} p(y | \mathcal{N}(\boldsymbol{x}_i)) l(\boldsymbol{h}_i^{\mathcal{N}}, y; \theta_h) \mathrm{d}y \leq \epsilon$ due to the upper bound of training error; and we have
\begin{equation}
\int_{\mathcal{Y}} p(y | \mathcal{N}(\boldsymbol{x}_i)) | l(\boldsymbol{h}_i, y; \theta_h) - l(\boldsymbol{h}_i^{\mathcal{N}}, y; \theta_h) | \mathrm{d}y \leq K_l d_{\boldsymbol{x}_i},
\end{equation}
due to the $K_l$-Lipschitz continuity of the loss function.

Taking Equations~(\ref{eq:upper_bound1}) and (\ref{eq:upper_bound2}) into Equation~(\ref{eq:error_upper_bound}), the generalization error on group $\boldsymbol{x} \in \mathcal{X}_a$ can be bounded by
\begin{equation}
\begin{aligned}
\int_{\mathcal{X}_a} \int_{\mathcal{Y}} p(\boldsymbol{x}, y) l(\boldsymbol{h}_i, y; \theta_h) \mathrm{d}y \mathrm{d}\boldsymbol{x} \leq \epsilon + (K_l + K_h L) \delta_a,
\end{aligned}
\end{equation}
where $\delta_a = \max_{\boldsymbol{x}_i \in \mathcal{X}_a} d_{\boldsymbol{x}_i} = \max_{\boldsymbol{x}_i \in \mathcal{X}_a} \min_{\boldsymbol{x}_j \in \mathcal{S}} || \boldsymbol{h}_i - \boldsymbol{h}_j ||_2$ denotes the max-min distance between the unannotated and annotated instances in the embedding space.
Note that $a \in \mathcal{A}$, we take $\tilde{a} = \arg \max_{a \in \mathcal{A}} \int_{\mathcal{X}_a} \int_{\mathcal{Y}} p(\boldsymbol{x}, y) l(\boldsymbol{h}, y; \theta_h) \mathrm{d}y \mathrm{d}\boldsymbol{x}$.
The generalization error difference between the two groups can be bounded by
\begin{equation}
\begin{aligned}
\label{eq:gap_bound2}
&\bigg| \int_{\mathcal{X}_0} \! \int_{\mathcal{Y}} \! p(\boldsymbol{x}, y) l(\boldsymbol{h}, y; \theta_h) \mathrm{d}\boldsymbol{x} \mathrm{d}y \!-\!\! \int_{\mathcal{X}_1} \! \int_{\mathcal{Y}} \! p(\boldsymbol{x}, y) l(\boldsymbol{h}, y; \theta_h) \mathrm{d}\boldsymbol{x}  \mathrm{d}y \bigg|
\leq \epsilon + (K_l + K_h L) \delta_{\tilde{a}}.
\end{aligned}
\end{equation}
Combine Equation~(\ref{eq:gap_bound2}) with (\ref{eq:gap_bound1}), we have the generalization error gap between group $x \in \mathcal{X}_0$ and group $x \in \mathcal{X}_1$ bounded as follow with probability $1-\gamma$,
\begin{equation}
\begin{aligned}
\label{eq:gap_bound3}
&\bigg| \int_{\mathcal{X}_0} \! \int_{\mathcal{Y}} \! p(\boldsymbol{x}, y) l(\boldsymbol{h}, y; \theta_h) \mathrm{d}\boldsymbol{x} \mathrm{d}y -\!\! \int_{\mathcal{X}_1} \! \int_{\mathcal{Y}} \! p(\boldsymbol{x}, y) l(\boldsymbol{h}, y; \theta_h) \mathrm{d}\boldsymbol{x}  \mathrm{d}y \bigg|
\\
&\quad\quad\quad\quad\quad\quad\quad\quad\quad\quad\quad\quad \leq \epsilon + \min \Big\{ \sqrt{-L^2 \log \gamma (2N_{\tilde{a}})^{-1}}, (K_l + K_h L) \delta_{\tilde{a}} \Big\}.
\nonumber
\end{aligned}
\end{equation}
\end{proof}

\begin{table*}
\centering
\caption{Details about the datasets.}
\tiny
\begin{tabular}{l|c|c|c|c|c}
    \hline
     & Adult
     & MEPS 
     & Loan default
     & German credit
     & CelebA \\
    \hline
     Domain & Social & Medical & Financial & Financial & Social \\
     Data formate & Tabular & Tabular & Tabular & Tabular & Image \\
     Predicted attribute & Salary &  Utilization & Defaulting & Credit & Wavy hair, Young \\
     Sensitive attribute & Gender & Race & Age & Age & Gender \\
     Number of instance & 30162 & 15830 & 30000 & 4521 & 5000 \\ 
     Number of attribute & 13 & 41 & 8 & 16 & 160$\times$160 \\
     Train, Validate, Test splitting & 0.25, 0.25, 0.5 & 0.25, 0.25, 0.5 & 0.25, 0.25, 0.5 & 0.25, 0.25, 0.5 & 0.25, 0.25, 0.5 \\
     Annotation budget & 4\textperthousand & 8\textperthousand & 4\textperthousand & 2\% & 3\% \\
    \hline
\end{tabular}
\label{tb:dataset_info}
\end{table*}

\begin{table*}
\centering
\caption{Detailed hyper-parameter setting.}
\tiny
\label{tb:hyperparameter_setting}
\begin{tabular}{l|c|c|c|c|c|c}
    \hline
     & Adult & MEPS & Loan default & German Credit & CelebA-hair & CelebA-young \\
    \hline
     Classifier body $f_b$ & Perceptron & Perceptron & Perceptron & Perceptron & ResNet-18 & ResNet-18 \\
     Classifier head $f_h$ & 2-layer MLP & 3-layer MLP & 2-layer MLP & 3-layer MLP & 3-layer MLP & 3-layer MLP \\
     Classifier head $f_a$ & Perceptron & Perceptron & Perceptron & Perceptron & Perceptron & Perceptron \\
     Embedding dim $M$ & 64 & 32 & 64 & 32 & 256 & 256 \\
     Hidden-layer dim & 32 & 32 & 32 & 32 & 64 & 128 \\
    \hline
\end{tabular}
\label{tb:hyperpara}
\end{table*}

\section{Details about the Datasets}

\label{sec:dataset_appendix}


The experiments are conducted on the MEPS\footnote{\scriptsize \url{https://github.com/Trusted-AI/AIF360/tree/master/aif360/data/raw/meps}}, Loan default\footnote{\scriptsize \url{https://archive.ics.uci.edu/ml/datasets/default+of+credit+card+clients}}, German credit\footnote{\scriptsize \url{https://archive.ics.uci.edu/ml/datasets/statlog+(german+credit+data)}}, Adult\footnote{\scriptsize \url{http://archive.ics.uci.edu/ml/datasets/Adult}} and CelebA\footnote{\scriptsize \url{http://mmlab.ie.cuhk.edu.hk/projects/CelebA.html}} datasets to demonstrate the proposed framework is effective to mitigate the socially influential bias such as the gender, race or age bias.
The statistics of the datasets is given in Table~\ref{tb:dataset_info}.
The details about the datasets including the size and spliting of the datasets, the predicted and sensitive attributes, and the annotation budget are described as follows.

\begin{itemize}[leftmargin=10pt]

\item[$\bullet$]  \textbf{MEPS}:
The task on this dataset is to predict whether a person would have a \emph{high} or \emph{low} utilization based on other features~(\emph{region, marriage, etc.}).
The \emph{Race} of each person is the sensitive attribute, where the two sensitive groups are \emph{white} and \emph{non-white}~\cite{cohen2002medical}.
The vanilla model shows discrimination towards the non-white group.
The annotation budget is 8\textperthousand~\footnote{The cost of annotating 8\textperthousand\ of training instances is affordable.}. 

\item[$\bullet$]  \textbf{Loan default}:
The task is to predict whether a person would default the payment of loan based on personal information~(\emph{Bill amount, education, etc.}), where the sensitive attribute is \emph{age}, and the two sensitive groups are people \emph{above 35} and those \emph{below 35}~\cite{bellamy2018ai}.
The vanilla trained model shows discrimination towards the younger group.
The annotation budget is 4\textperthousand.

\item[$\bullet$]  \textbf{German credit}: 
The goal of this dataset is to predict whether a person has \emph{good} or \emph{bad} credit risks based on other features~(\emph{balance, job, education, etc.}).
\emph{Age} is the sensitive attribute, where the two sensitive groups are people \emph{older than 35} and those \emph{not older than 35}~\cite{UCI:2017}.
The vanilla trained model shows discrimination towards the younger group.
The annotation budget is 2\%.

\item[$\bullet$]  \textbf{Adult}:
The task for this dataset is to predict whether a person has \emph{high} (more than 50K/yr) or \emph{low} (less than 50K/yr) income based on other features~(\emph{education, occupation, working hours, etc.}).
\emph{Gender} is considered as the sensitive attribute for this dataset~\cite{UCI:2007}.
Thus, we have two sensitive groups \emph{male} and \emph{female}.
The vanilla trained classification model shows discrimination towards the female group.
The annotation budget is 4\textperthousand.

\item[$\bullet$]  \textbf{CelebA}:
This is a large-scale image dataset of human faces~\cite{liu2015deep}. 
We consider two tasks for this dataset:
i) identifying whether a person has wavy hair; 
ii) identifying whether a person is young.
\emph{Gender} is the sensitive feature, where the two sensitive groups are \emph{male} and \emph{female}. 
The vanilla trained model shows discrimination towards male in task i) and female in tasks ii), respectively.
The annotation budget is 3\%.

\end{itemize}

\section{Implementation Details}

\label{sec:implement_appendix}

The experiment on each dataset follows the pipeline of \textbf{pre-training}, \textbf{debiasing}, and \textbf{head-selection}.
Each step is shown as follows.

\noindent
\textbf{Pre-training}: We pre-train $f_b (\bullet \mid \theta_b)$ to minimize the contrastive loss on the whole training set without any annotations for 50 epochs; and pre-train $f_h( \boldsymbol{h} \mid \theta_h )$ for 10 epochs to minimize the cross-entropy $\text{CE}(\hat{y}, y)$; then pre-train $f_b (\boldsymbol{h} \mid \theta_b)$ for 10 epochs to minimize the cross-entropy $\text{CE}(\hat{a}, a)$, where the initial sensitive annotations are very few~(less than 10), randomly selected from each group.
$\theta_b, \theta_h$ and $\theta_a$ provide initial solutions for the bias mitigation.

\noindent
\textbf{Debiasing}: We adopt \Algnameabbr{} to debias the classifier head $f_h (\bullet \mid \theta_h)$ for several iterations.
Specifically, the number of iterations equals the available annotation number, where \Algnameabbr{} selects one instance for annotation, debiases $f_h (\bullet \mid \theta_h)$ and retrains $f_a (\bullet \mid \theta_a)$ for 10 epochs in each iteration, and back up the checkpoint of $\theta_h$ and $\theta_a$ in the last epoch of each iteration.
In the Pre-training and Debiasing stages, the parameters $\theta_b, \theta_h$ and $\theta_a$ are updated using the Adam optimizer with a learning rate of $10^{-3}$, mini-batch size 256 and a dropout probability of $0.5$.
The DNN architectures and detailed hyper-parameter settings on different datasets are given in Appendix~\ref{sec:hyper_param_appendix}.

\noindent
\textbf{Head-selection}: We use the trained $f_a(f_b (\bullet \mid \theta_b) \mid \theta_a)$ to generate the proxy sensitive annotations for the validation dataset so that the fairness metrics can be estimated on the validation dataset.
The optimal debiased classifier head $f_h$ is selected to maximize the summation of accuracy and fairness score on the validation dataset.
We merge the selected $f_h$ with the pre-trained $f_b$ and test the classifier $f_h(f_b (\bullet \mid \theta_b) \mid \theta_h)$ on the test dataset.
This pipeline is executed five times to reduce the effect of randomness, and the average testing performance and the standard deviation are reported in the remaining sections.

\section{Detailed Hyper-parameter Setting}

\label{sec:hyper_param_appendix}

The detailed hyper-parameter setting is given in Table~\ref{tb:hyperparameter_setting}.

\section{Details about the Baseline Methods}
\label{sec:baseline_appendix}

We introduce details on the baseline methods in this section.

\begin{itemize}[leftmargin=5pt]

    
    \item[$\bullet$]  \textbf{Group DRO}: Group DRO maintains a distribution $\boldsymbol{q} = [q_0, q_1]$ over the sensitive groups $a \in \mathcal{A}$, and updates the classifier $f(\bullet \mid \theta_f)$ via the min-max optimization given by
    \begin{equation}
    \theta_f = \arg \min_{\theta} \max_{\boldsymbol{q}} \sum_{a \in \mathcal{A}} \frac{q_a}{N_a} \sum_{(\boldsymbol{x}_i, y_i) \sim \mathscr{D}_a} l(\boldsymbol{x}_i, y_i; \theta),
    \end{equation}
    where $\mathscr{D}_a = \{ \boldsymbol{x}_i, y_i \mid a_i = a \}$ depends on fully-annotated training set to generate the sensitive groups.
     
    
    \item[$\bullet$]  \textbf{FAL}: Original FAL depends on the annotation of sensitive attribute to have active instance selection.
    Hence, we consider an improved version of the original framework to adapt to the problem in this work.
    Specifically, our improved FAL updates the classifier to minimize the cross-entropy loss on the annotated dataset.
    The annotated instances are selected by
    \begin{equation}
    (\boldsymbol{x}^*, y^*) = \arg \!\! \max_{(x,y) \in \mathscr{U}} \!\! \alpha \text{ACC} (f_t) + (1 \!-\! \alpha) [ \mathcal{F}(f_t) \!-\! \mathcal{F}(f_{t-1}) ],
    \end{equation}
    where $f_t$ denotes the classifier learned on the annotated dataset $\mathcal{S}$; $\mathcal{F}(f_{t})$ denotes the fairness score of classifier $f_{t}$, which is the value of Equalized Odds in our experiment; $\alpha$ controls the trade-off between accuracy and fairness; and we have $\alpha$ in the range of $[0.5, 1]$ in our experiments.
    
    
    \item[$\bullet$]  \textbf{LfF}: LfF adopts generalized cross entropy loss to learn the biased model $f_B$ to provide proxy annotation, and simultaneously learn the debiased model $f_D$ towards  minimizing the cross entropy re-weighted by the proxy annotation.
    $f_B$ and $f_D$ are updated by
    \begin{equation}
    \begin{aligned} 
    \theta_{B}^* &= \min_{\theta_B} \sum_{i=1}^N \frac{1-p(\boldsymbol{x}_i; \theta_B)^q}{q},
    \\
    \theta_D^* &= \arg \min_{\theta_D} \sum_{i=1}^N \frac{l(\boldsymbol{x}_i, \hat{y}; \theta_B) l(\boldsymbol{x}_i, \hat{y}; \theta_D)}{l(\boldsymbol{x}_i, \hat{y}; \theta_B) + l(\boldsymbol{x}_i, \hat{y}; \theta_D)},
    \end{aligned} 
    \end{equation}
    where we control the hyper-parameter $q$ in the range of $[2.5, 3]$ in our experiments.
    
    \item[$\bullet$]  \textbf{SSBM}: This method initially randomly select a subset for annoatation, then adopts POD for the bias mitigation.
    
    \item[$\bullet$]  \textbf{POD+RS}: Different from SSBM, this method randomly selects an annotated instance and adopts POD for bias mitigation in each iteration.
    The random instance selection and POD executes iteratively.
    This method is designed for studying the effect of annotation ratio to the mitigation performance.
    

    \item[$\bullet$]  \textbf{POD+AL}: This method adopts POD for bias mitigation.
    Different from \Algnameabbr{}, the annotated instances are selected by uncertainty sampling.
    Specifically, we calculate the Shannon entropy of the model prediction for each instance in the unannotated dataset.
    For $\boldsymbol{x}_i \in \mathscr{U}$, we have the entropy given by
    \begin{equation}
        \mathscr{H}(\boldsymbol{x}_i) = -p_{\hat{y}_i=1} \log_2 p_{\hat{y}_i=1} - p_{\hat{y}_i=0} \log_2 p_{\hat{y}_i=0}.
    \end{equation}
    where $[p_{\hat{y}_i=1}, p_{\hat{y}_i=0}] = \text{softmax} [ f(\boldsymbol{h}_i | \theta_h) ]$; and $f(\boldsymbol{h}_i | \theta_h) \in \mathcal{Y}$.
    The instance for annotation is selected by 
    \begin{equation}
        (\boldsymbol{x}^*, y^*) = \arg \max_{(\boldsymbol{x}_i, y_i) \in \mathscr{U}} \mathscr{H}(\boldsymbol{x}_i).
    \end{equation}

    \item[$\bullet$]  \textbf{POD+CA}: This method adopts POD for bias mitigation.
    Different from \Algnameabbr{}, POD+CA selects the instance for annotation following the max-min rule given by
    \begin{equation}
    (\boldsymbol{x}^*, y^*) = \arg \max_{\boldsymbol{x}_i \in \mathscr{U}} \min_{\boldsymbol{x}_j \in \mathcal{S}} || \boldsymbol{h}_i - \boldsymbol{h}_j ||_2,
    \end{equation}
    where $\mathcal{S}$ and $\mathscr{U}$ denote the annotated and unannotated datasets, respectively.

\end{itemize}

\end{document}